\documentclass[twoside]{article}
\usepackage{algorithm}
\usepackage{appendix}
\usepackage{graphicx, epsfig,color}
\usepackage{natbib}
\usepackage{amsmath}
\usepackage{amsthm}
\usepackage{amssymb}
\usepackage{amsfonts}
\usepackage{color}
\usepackage{booktabs}
\usepackage{float}
\usepackage{placeins}

\usepackage[accepted]{aistats2024}
\usepackage{amsmath}

\usepackage{amssymb}
\usepackage{amsthm}

\usepackage{microtype}
\usepackage{graphicx}
\usepackage{subfigure}
\usepackage{booktabs} 
\usepackage{natbib}
\usepackage{xcolor}
\usepackage{natbib}
\usepackage{hyperref}

\usepackage{url,bm}

\usepackage{amsthm,amsfonts,amsmath,amssymb,epsfig,color,float,graphicx,verbatim}
\usepackage{algorithm,algorithmic}
\usepackage{bbm}
\usepackage{enumitem}

\usepackage{graphicx}

\usepackage{array}

\newtheorem{theorem}{Theorem}[section]

\newtheorem{lemma}{Lemma}[subsection]
\newtheorem{corollary}{Corollary}[subsection]
\newtheorem{definition}{Definition}[section]

\renewcommand{\eqref}[1]{Eq.~(\ref{#1})}
\newcommand{\lemref}[1]{Lemma~\ref{#1}}

\usepackage{amssymb}

\usepackage{bbm}

\newcommand{\stam}[1]{}
\newcommand{\ignore}[1]{}
\usepackage{comment}

\usepackage{mathtools}

\newtheorem{assumption}[theorem]{Assumption}

\newcommand{\bx}{\mathbf{x}}
\newcommand{\bxt}{\Tilde{\bx}}
\newcommand{\bw}{\mathbf{w}}

\newcommand{\bu}{\mathbf{u}}

\newcommand{\bz}{\mathbf{z}}

\newcommand{\bM}{\mathbf{M}}

\newcommand{\bP}{\mathbf{P}}
\newcommand{\bI}{\mathbf{I}}
\newcommand{\bL}{\mathbf{L}}

\newcommand{\bmu}{\boldsymbol{\mu}}
\newcommand{\bmut}{\Tilde{\bmu}}
\newcommand{\bzeta}{\boldsymbol{\zeta}}

\newcommand{\bxi}{\boldsymbol{\xi}}
\newcommand{\bxit}{\Tilde{\bxi}}

\newcommand{\btheta}{{\boldsymbol{\theta}}}
\newcommand{\bomega}{\boldsymbol{\omega}}

\newcommand{\ex}{\epsilon_x}
\newcommand{\exi}{\epsilon_{\xi}}
\newcommand{\exxi}{\epsilon_{x \xi}}

\newcommand{\co}{{\cal O}}

\newcommand{\cl}{{\cal L}}

\newcommand{\cs}{{\cal S}}
\newcommand{\cn}{{\cal N}}

\DeclareMathOperator*{\sign}{sign}

\newcommand{\reals}{{\mathbb R}}

\DeclareMathOperator*{\E}{\mathbb{E}}

\newcommand{\inner}[1]{\langle #1 \rangle}

\newcommand{\norm}[1]{\left\|#1\right\|}
\newcommand{\snorm}[1]{\|#1\|} 

\begin{document}
\twocolumn[
\aistatstitle{Effect of Ambient-Intrinsic Dimension Gap on Adversarial Vulnerability}
\aistatsauthor{ Rajdeep Haldar \And Yue Xing \And Qifan Song}
\aistatsaddress{ Purdue University \And Michigan State University \And Purdue University}

]
\begin{abstract}
The existence of adversarial attacks on machine learning models imperceptible to a human is still quite a mystery from a theoretical perspective. In this work, we introduce two notions of adversarial attacks: natural or on-manifold attacks, which are perceptible by a human/oracle, and unnatural or off-manifold attacks, which are not. We argue that the existence of the off-manifold attacks is a natural consequence of the dimension gap between the intrinsic and ambient dimensions of the data. For 2-layer ReLU networks, we prove that even though the dimension gap does not affect generalization performance on samples drawn from the observed data space, it makes the clean-trained model more vulnerable to adversarial perturbations in the off-manifold direction of the data space. Our main results provide an explicit relationship between the $\ell_2,\ell_{\infty}$ attack strength of the on/off-manifold attack and the dimension gap.
\end{abstract}

\section{Introduction}
Neural networks exhibit a universal approximation property, enabling them to learn any continuous function up to arbitrary accuracy\citep{hornik1989multilayer,pinkus1999approximation, leshno1993multilayer}. This has been used to address complex classification problems by learning their decision boundaries. However, although neural networks are known to generalize well \citep{liu2019understanding}, they are vulnerable to certain small-magnitude perturbations, hampering classification accuracy, which are known as \emph{adversarial examples} \citep{szegedy2013intriguing,goodfellow2014explaining,carlini2017towards,madry2017towards}. 
\begin{figure}[h]
\includegraphics[height=4.8cm,width=8.5cm]{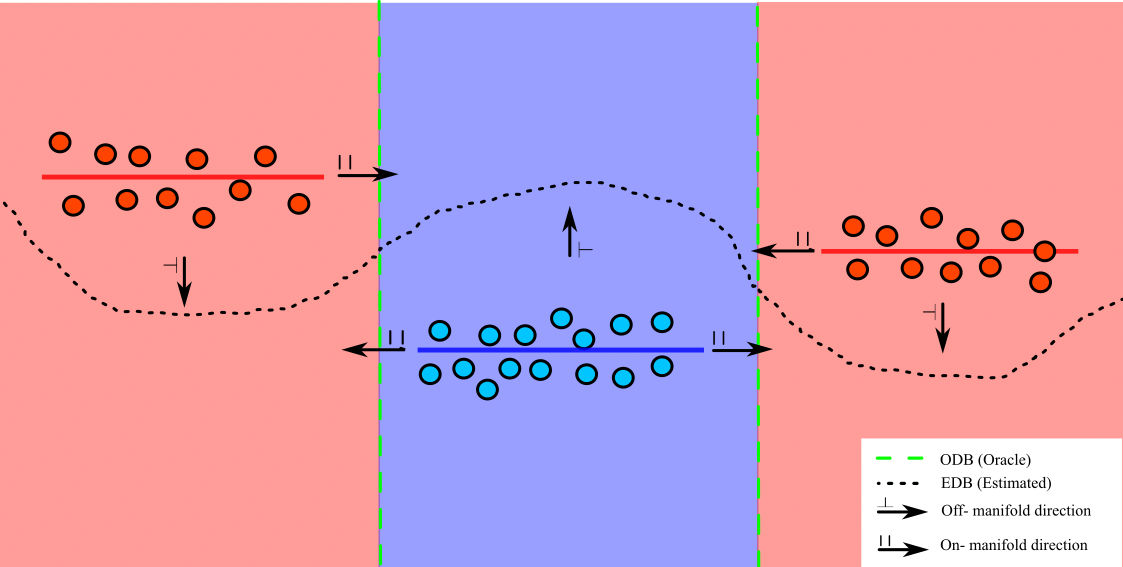}
\caption{Mental image: The oracle decision boundary (green dashed line) determines the label (blue or red) of any point in the Euclidean space. The observed data space consists of 1-dimensional line segments immersed in the 2-dimensional space. The model learns the estimated decision boundary (black dotted line) based on the observed data. 
}
\label{fig:mental image}
\end{figure}

For a classification problem, assume that there exists an \emph{oracle decision boundary} (ODB), which is the human decision boundary 
that determines the labels for any lawful input (e.g., a consensus digit recognition rule for any arbitrary gray-level images).
The goal of training a neural network classifier is to obtain an \emph{estimated decision boundary} (EDB) based on observed data (sample-label pairs observed). To facilitate misclassification by the neural network, a perturbation must push the original sample across the EDB.

Different types of attacks can be defined when the data is only observable in a particular subspace of the lawful input space (the data space). Usually, the EDB can estimate the ODB within the data space reasonably, given a large sample size. Thus, an adversarial perturbation restricted to the data space, which aims to cross over the ODB, tends to cross over the EDB as well. We call these adversarial attacks \emph{natural} or \emph{on-manifold} attacks. Even an oracle will be prone to wrong classifications with these adversarial perturbations.

In contrast, outside the data space, where no observation is made, the EDB curve is merely an extrapolation extended from the data space. It will not necessarily approximate the ODB well. Consequently, an adversarial perturbation outward from the data space, facilitated by crossing the EDB, will not necessarily try to push the sample over the ODB. We call these perturbations \emph{unnatural} or \emph{off-manifold} attacks. Under these attacks, an oracle can accurately distinguish and classify the perturbed sample, but the trained model fails. Figure \ref{fig:mental image} is provided to illustrate this phenomenon with a simple example. \cite{shamir2022dimpled}'s observations also support the phenomenon we describe, where the EDB learned by neural networks, over the training process, tends to align with the data space first and then form dimples around differently labeled clusters to attain correct classification on the observed data space, leaving them vulnerable to examples outside the data space.

\begin{figure}[h]
\centering\vspace{-0.1in}
\includegraphics[scale=0.25]{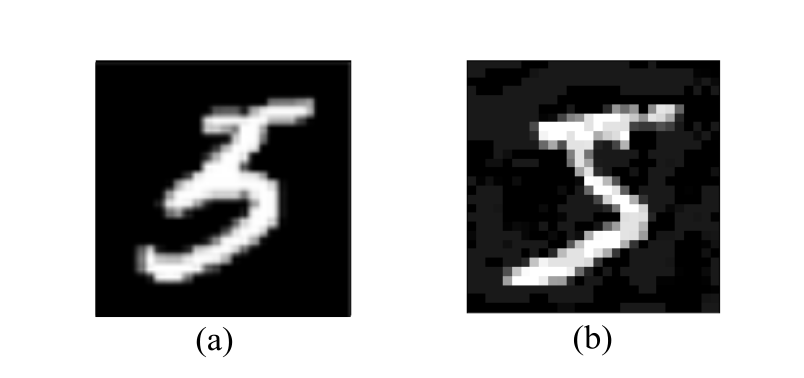}
\caption{Original image is a digit 5. (a) Natural attack; (b) Unnatural attack (minor changes in pixel values of the background compared to the original). Under both of these attacks, the neural network classifies the image as 3 instead of 5.}
\label{fig:motivation}
\end{figure}
Figure \ref{fig:motivation} shows an example of an \emph{unnatural} and \emph{natural} attack on a digit 5 image in the MNIST dataset \citep{lecun2010mnist}. The unnatural attacked image can still be classified as 5 by a human/oracle; however, the model classifies it as digit 3. The vulnerability of neural networks to \emph{unnatural} attacks is a pressing issue.\\

In this work, we obtain distinguished results for off-manifold and on-manifold attacks. Our results imply that unnatural/off-manifold attacks are fundamental consequences of the dimension gap between the local intrinsic dimension of the observed data space and the ambient dimension. That is, 
\begin{theorem}[Informal version of Theorem \ref{thm:pert flips}]
    As the dimension gap increases, the strength of $\ell_2,\ell_{\infty}$ attack required to misclassify the sample decreases, i.e., the model is more vulnerable. Notably, for $\ell_{\infty}$ attack, the strength asymptotically goes to $0$ w.r.t. the dimension gap. 
\end{theorem}
We use intrinsic dimension as an umbrella term for well-known concepts of \emph{Manifold dimension, Minkowski dimension, Hausdorff dimension} etc. In our theoretical setting (Section \ref{subsec: data}), all of these notions of dimension are equivalent and refer to the signal residing in a lower dimensional set locally.

\begin{table}[h]
\begin{tabular}{@{}lllll@{}}
\toprule
         & Ambient & LPCA & MLE(k=5) & TwoNN \\ \midrule
MNIST    & 784     & 38   & 13.36    & 14.90 \\
CIFAR-10 & 3072    & 9    & 27.66    & 31.65 \\ \bottomrule
\end{tabular}
\caption{Intrinsic dimension estimation}
\label{tab:intrinsic_dimensions}
\end{table}
To fuel the reader's motivation, Table \ref{tab:intrinsic_dimensions} is presented to show the estimated local intrinsic dimension for the standard image datasets MNIST and CIFAR-10. Regardless of the estimation method (local PCA\citep{Fukunaga1971}, maximum likelihood estimation\citep{Haro2008}, Two nearest-neighbor\citep{Facco_2017}), the ambient dimension of the images is way larger than the intrinsic dimension of the space they reside in. This paper reveals that this is the root cause of the existence of \emph{unnatural attacks}. 
Furthermore, as implied by our theorem, a larger ambient-intrinsic dimension gap, e.g., training images with unnecessarily high resolution but no more useful details, leads to higher vulnerability.

\section{Related Works}
Besides the literature, as mentioned earlier, below is a list of works ordered by relatability to this paper.
\paragraph{Adversarial study motivated by a manifold view}
\citep{xiao2022understanding} assumes a manifold view of the data and focuses on a generative model to learn the data manifold and conceive on-manifold attacks in practice. \citep{zhang2022manifold} conducts a theoretical study to decompose the adversarial risk into four geometric components: the standard risk, nearest neighbor risk, in-manifold, and normal (off-manifold in our setting) risk. Our paper also subscribes to the manifold hypothesis; Compared to the above two studies, our focus and mathematical setting are on the relationship between the data dimension gap and adversarial attack strength. 

\paragraph{Adversarial examples in random ReLU networks}
Different from this paper which studies trained neural networks, some other studies, e.g., \cite{daniely2020relu,bubeck2021single} also show the existence of $\ell_2$ adversarial examples in two-layer ReLU networks with randomized weights and provide a corresponding rate for the strength of the $\ell_2$ perturbation. 
\cite{bartlett2021adversarial} further generalizes to a multi-layer network. 
\paragraph{Adversarial examples due to implicit bias in two-layer Relu networks}
\cite{vardi2022gradient} and \cite{frei2023double} argue that adversarial examples are a consequence of the implicit bias in clean training despite the existence of robust networks. Extending from their work, we introduce the notion of the intrinsic-ambient dimension gap exhibited by the data, which gives rise to two different types of attacks \emph{natural} or \emph{on-manifold} and \emph{unnatural} or \emph{off-manifold}. The previously mentioned works can be reduced to the \emph{on-manifold setting} in our framework.
 
\paragraph{Adversarial examples due to low dimensional data}
 A recent work, \cite{melamed2023adversarial}, studies the adversarial robustness of clean-trained neural networks with data in a low-dimensional linear subspace and argues that the neural networks are vulnerable to attacks in the orthogonal direction. 
 
 Though the motivation is similar, our findings and outcomes are different. Detailed comparisons of are provided in Section \ref{remark:setting} for the model setup, and Section \ref{remark:rate_comparison} for the final result, respectively. In short, the key differences arise from their setting of linear subspace, lazy training, and weight initialization. Though it is worth consideration in the finite training regime, the adversarial vulnerability considered in their setting does not hold if the neural network training attains convergence. On the other hand, our results are invariant to initialization, and our theoretical setting is more realistic than prior works.

\section{Problem Setup}
In this section, we introduce the detailed technical settings considered in this paper. We use consistent notations with \cite{frei2023double} in this paper. For any natural number $n$, the set $\{1,\ldots, n\}$ is represented by the shorthand $[n]$. For any vector $\bu$, $\norm{\bu}$ will denote the euclidean $\ell_2$ norm. The notations $\co{}$ and $\Theta$ are the standard asymptotic notations.
\subsection{Data}
\label{subsec: data}
In this paper, we study a binary classification problem with $k$ clusters with respective labels $y^{(r)}\in \{-1,1\},\,r\in [k]$ and $n$ observations.
The intrinsic signal lies in $d$ dimensions, and we observe some transformation of the signal in $D$ dimensions ($D\geq d$). The dimension gap is denoted by $g=D-d$.
Let $\bx_i=\bmu^{(r)}+\bxi_i \in \mathbb{R}^d$ denote the $i^{th},\, i\in [n]$ intrinsic signal for the $r^{th}$ cluster. The vectors $\bmu^{(1)},\ldots,\bmu^{(k)}\in \mathbb{R}^d$ are the means of the $k$ clusters, and $\bxi_i\in \mathbb{R}^d$ is the deviation of the signal around the mean.

In terms of the actual observed data, denote the $i^{th}$ observation $\bxt_i=\bM\bx_i+\bzeta^{(r)}+\bomega_i\in \mathbb{R}^D$,
where $\bM$ is a $D\times d$ matrix with $d$- orthonormal columns; the vector $\bzeta^{(r)}\in \mathbb{R}^D$ is the additional random cluster effect and $\bzeta^{(r)}\sim N(0,\tau^2\bI_D)$ for all $r\in[k]$; and $\bomega_i$ represents the ambient noise. We will also use an alternate representation for the $i^{th}$ observation in terms of the transformed mean and deviation term in higher dimensions, $\bxt_i=\bmut^{(r)}+\bxit_i$ where $\bmut^{(r)}=\bM\bmu^{(r)}$ and $\bxit=\bM\bxi_i+\bzeta^{(r)}+\bomega_i$.
The terms $\bxi_i,\bomega_i$ are also random, and we do not impose any specific distribution for them; instead, we will restrict their norm for our analysis.
The notation $\cs=\{(\bxt_i,y_i)\}_{i=1}^n$ is the collection of all $n$ observation-label pairs used to train our model.

\subsection{Data Interpretation}
\label{subsec: interpret}
When studying low-dimensional information embedded in a higher-dimensional space, we do not want the original information to be lost or distorted. We restrict our analysis to a class of transformations known as \emph{Isometric} Transformations. \emph{Isometric} transformations are essential as they preserve distances (structure/information) between any pair of points in the origin space.

 Any \emph{isometric} transformation can be expressed as an affine transformation with a linear part consisting of orthogonal columns and a geometric translation.  Hence, $\bM(\cdot)+\bzeta^{(r)}$ preserves distance within each cluster. The matrix $\bM$ applies a linear transformation on the $d$-dimensional intrinsic signal ($\bx_i$), immersing it into a $D$-dimensional space. After the immersion, a translation or shift $\bzeta^{(r)}$ is added corresponding to the $r^{th}$ cluster. Then an ambient noise $\bomega$ is added to account for the stochasticity in the observation process. Each cluster can be considered a lower $d$ dimensional manifold immersed in a $D$ dimensional space up to ambient noise. Our data space is a union of such manifolds, and our setup can be used as a motivation for real-life datasets that locally reside in a lower dimensional manifold. Examples can be found in Tables \ref{tab:intrinsic_dimensions}, \ref{tab:cifar10 individual classes}, \ref{tab:cifar10 individual classes full}, \ref{tab:mnistclasses}.
 
\begin{table}[h]
\begin{tabular}{llllll}
\hline
CIFAR class & 0     & 1     & $\dots$ & 8     & 9     \\ \hline
LPCA        & 8     & 11    & $\dots$ & 10    & 16    \\
MLE (k=5)   & 18.34 & 21.20 & $\dots$ & 19.80 & 24.42 \\
TwoNN       & 21.77 & 25.98 & $\dots$ & 24.77 & 29.11 \\ \hline
\end{tabular}
\caption{Intrinsic dimension of CIFAR-10 Dataset. The intrinsic dimension over all classes is similar to that for individual classes.
(Complete version: Table \ref{tab:cifar10 individual classes full} in the appendix.)\label{tab:cifar10 individual classes}} 
\vspace{-1.8mm}
\end{table}
Within each cluster $r$, the distances are preserved by the transformation $\bM(\cdot)+\bzeta^{(r)}$ up to the variance introduced by the ambient noise ($Var(\bomega)$). Between clusters, the distances are preserved by $\bM(\cdot)$ up to a slightly larger variance due to cluster translation and the noise term $(2\tau^2+Var(\bomega))$. Intuitively, our setup transforms a low-dimensional signal into a high-dimensional signal via $\bM$, preserving all information (pairwise angle and distances) up to some variance, where the within-cluster variance is lower than the between-cluster variance.

\subsection{Neural Network}
In this work, we will use a 2-layer ReLU network $\cn_{\btheta}:\mathbb{R}^D\to \mathbb{R}$ of width $w$ as our classification model.
\begin{equation}
    \cn_{\btheta}(\bxt)=\sum_{j=1}^wv_j\phi(\bw_j^\top\bxt+b_j)
\end{equation}
$\bw,b$ represent the weights and biases of the first layer; $v_j$ is a scalar weight of the second layer for the $j^{th}$ neuron. The function $\phi(x)=\max(0,x)$ is the ReLU activation function. The vectorized collection of all the parameters $(\bw_1^T,\ldots,\bw_j^T,b_1\ldots,b_j,v_1,\ldots,v_j)^T$ for the neural network is represented by $\btheta$. Both the two layers will be updated during the training.

\subsection{Expression for Network Parameters}
Before we proceed, we need to establish the training loss used in our framework. For our neural network model $\cn_{\btheta}$ parameterized by $\btheta$, the \emph{empirical loss} $\cl({\btheta})$ can be defined on the dataset $\cs=\{(\bxt_i,y_i)\}_{i=1}^n$ as follows:
\begin{equation}
\label{eq:loss}
	\cl(\btheta) := \frac{1}{n} \sum_{i=1}^n \ell(y_i \cdot\cn_{\btheta}(\bxt_i))~.
\end{equation} 
When $\ell(z)=e^{-z}$ or $\ell(z)=\log (1+e^{-z})$, $\cl(\btheta)$ is the empirical \emph{exponential loss} or \emph{logistic loss} respectively.

Some \emph{implicit bias} literature studies the limiting behavior of the training process. Based on \cite{lyu2020gradient},
two-layer ReLU networks have directional convergence to the solution of the \emph{maximum margin problem} \citep{james2013introduction}:
\begin{theorem}[\cite{lyu2020gradient}]
\label{thm:kkt}
    Consider $\cn_{\theta}(\bxt)$ trained via gradient descent or gradient flow for the exponential or logistic loss. If for some time $t_0$ in the training process, the loss reaches a sufficiently small value $\cl(\theta(t_0))<\frac{1}{n}$. Then, under some regularity and smoothness assumptions, $\frac{\btheta}{\snorm{\btheta}}$converges to the KKT solution of the maximum margin problem:
    \begin{equation}
    \label{eq:max_margin}
        \min_\btheta \frac{1}{2} \norm{\btheta}^2 \;\;\;\; \text{s.t. } \;\;\; \forall i \in [n] \;\; y_i \cn_{\btheta}(\bxt_i) \geq 1~.
    \end{equation}
   Furthermore, $\cl(\btheta(t)) \to 0$ and $\norm{\btheta(t)} \to \infty$ as $t \to \infty$.
\end{theorem}
Theorem \ref{thm:kkt} shows the limiting behavior of the gradient flow in clean training and generally applies to any homogeneous network defined in \cite{lyu2020gradient}.
A consequence of the above theorem is that we can express the parameters in terms of the data, i.e., $\btheta=\sum_i^n\lambda_iy_i\nabla_{\btheta}\cn_{\btheta}(\bxt_i)$, where $\lambda_i\geq 0$ are the Lagrangian multipliers. Specifically,
\begin{equation}
        \label{eqn: main kkt}
    \bw_j=\sum_{i\in [n]}v_j\lambda_iy_i\phi'_{i,j}\bxt_i; \hspace{0.3cm}    b_j=\sum_{i\in [n]}v_j\lambda_i y_i \phi'_{i,j}
\end{equation}
$\phi'_{i,j}$ denotes the sub gradient of $\phi(\bw_j^\top\bxt_i+b_j)$.
\subsection{Volatile biases}
\label{subsec: volatile bias}
In the previous section, we saw that the network parameters can be explicitly written in terms of the observed dataset. In this section, we will leverage that closed-form expression to justify the existence of vulnerable weight components responsible for \emph{unnatural attacks}.

Let $\bP_o=\bM(\bM^{\top}\bM)^{-1}\bM^{\top}=\bM\bM^{\top}$ be the projection matrix of the image space where the intrinsic signal is initially immersed. On the contrary, $\bP=\bI_D-\bP_o$ is the projection onto the co-kernel, where no component of the intrinsic signal $\bx$ resides. We will associate the terms \emph{on-manifold} and \emph{off-manifold} to the spaces dictated by $\bP_o$ and $\bP$, respectively.

In the neural network, each weight $\bw_j$ can be decomposed into an \textit{authentic-weights} part $\bP_o\bw_j$ and a \textit{volatile biases} part $\bP\bw_j$. Multiplying, $\bP_o,\bP$ to \eqref{eqn: main kkt} we get:
\begin{align}
\label{eq: main kkt condition volatile w}
\bP_o\bw_j & =  \sum_{i \in [n]} \lambda_i y_i v_j \phi'_{i,j} \left[\bM\bx_i+\bP_o\left(\bzeta^{(r)}+\bomega_i\right)\right] \nonumber\\
	\bP\bw_j &=  \sum_{i\in I^{(r)}}\sum_{r \in [k]} \lambda_i y_i v_j \phi'_{i,j} \bP\left(\bzeta^{(r)}+\bomega_i\right),
\end{align}
where $I^{(r)}$ is the index set of samples in the $r^{th}$ cluster. The authentic weights are the ones that carry all the information of the intrinsic signal $\bx$ from the lower $d$ dimension space; the volatile biases are dummy weights (independent of $\bx$) that act just as additional bias terms during the \emph{clean training}. Any interaction between the immersed intrinsic signal and volatile biases is zero as $\bP\bM=0$. However, they can be exploited by \emph{off-manifold} attacks to change the neuron output. 
\paragraph{Comparison to \cite{melamed2023adversarial}}
\label{remark:setting}
In \cite{melamed2023adversarial}, the setting is equivalent to a linear transformation of a low-dimensional intrinsic signal to a higher dimension, i.e., under our framework, $\bxt=\bL\bx$ where $\bL$ is any $D\times d$ matrix. 
 Let $\pi^{\perp}=\bI_D-\bL(\bL^\top\bL)^{-1}\bL^\top$ be the projection matrix. Using Theorem \ref{thm:kkt} and \eqref{eqn: main kkt} in their setting, we have $\pi^{\perp}\bw_j=0\,, j\in [w]$. 
 
 The consequence is that, based on results in Theorem \ref{thm:kkt}, at convergence, the orthogonal weights converge to 0; hence, these weights have no contribution to the neuron output. The adversarial vulnerability described in \cite{melamed2023adversarial} results from finite training and large weight initialization. In real practice, since one can only train finite iterations, the initialization property observed by \cite{melamed2023adversarial} is still a factor to consider. However, it is not the driving force for the existence of \emph{unnatural attacks}. We also validate these arguments with experiments on simulated and real-life datasets (Section \ref{sec: init}).

\subsection{Assumptions}
\label{subsec: assump}
In this section, we introduce some additional assumptions to simplify the analysis.
\begin{assumption} \label{ass:dist}
    We have:
    \begin{enumerate}[label=(A\arabic*)]
    \item \label{a1} $\norm{\bmu^{(j)}} = \sqrt{d}$ for all $r \in [k]$.
    \item $k \left( \max_{i \neq j} |\inner{\bmu^{(i)}, \bmu^{(j)}}| + \Delta' + 1 \right) = c'(d  - \Delta'+ 1)$.\,; $c'<1/10$ 
     \item \label{a2}$\norm{\bxi}\leq \sqrt{2} \ln(d); \norm{\bomega}\leq 1$
    \item $c_3(g,\tau,k)=\frac{g\tau^2}{20k}-\frac{10k+1}{10k}\sqrt{\frac{13g\tau^2}{8}}>0;$ \\ $
        c_5(d,\tau,k)=\frac{d\tau^2}{20k}-\frac{10k+1}{10k}\sqrt{\frac{13d\tau^2}{8}}(\sqrt{2\ln d +1})\\>0;$
    \end{enumerate}
    where $\Delta'=\co{\left(\max \{\sqrt{d}\ln d,D\tau^2\}\right)}$  (Definition \ref{def:constants}).
\end{assumption}
Assumptions $(A1), (A2)$ are similar to the assumptions used in \cite{frei2023double}. Condition $(A1)$ states that the signal strength must grow with the intrinsic dimension. It can be generalized by replacing $\sqrt{d}$ to $\Theta{(\sqrt{d})}$ without changing the analysis in any meaningful way.  Condition $(A2)$ ensures that the means are near orthogonal to each other; we impose that pairwise interaction between the means is less than $\norm{\bmu^{(r)}}^2=d$ for any cluster $r$. This assumption can be satisfied for some $\tau^2=\co{(d/D)}$.

For the other two assumptions, $(A3)$ are bounds on the stochastic deviation and error terms. In general, these bounds can be relaxed to $\co{(\ln d)}, \co{(1)}$ respectively. These bounds are easy to satisfy for sufficiently large $d$, for instance, $\bxi\sim N(0,\bI_d); \bomega\sim N(0,D^{-1}\bI_D)$ satisfy the above bounds with high probability. In addition, Condition $(A4)$ is a mild assumption. The order of the positive terms is higher than the negative terms for both $c_3(g,\tau,k)$ and $c_5(d,\tau,k)$. A constant can always be incorporated inside $\tau$ to satisfy the assumption.
\section{Main Results}

In this section, we will state our main results. We study the adversarial vulnerability for any reasonable testing samples that arise from the observable data space (union of manifolds) and are homogeneous with training samples. Formally, we define such examples as \emph{nice examples} \ref{def: nice example}. 
These examples can be expressed as the transformation of some $d$ dimensional intrinsic signal described in Section \ref{subsec: data}. 
 \begin{definition}[Nice example]
 \label{def: nice example}
     An $\bxt \in \mathbb{R}^D$ is nice if it can be expressed as $\bxt=\bM\bx+\bzeta^{(r)}+\bomega$, for some $r\in [k]$ and $\bx=\bmu^{(r)}+\bxi$ such that $\norm{\bxi}\leq \sqrt{2}\ln(d)$, $\norm{\bomega}\leq 1$.
 \end{definition}
 In the following theorem, we show that for any \emph{nice-example}, the clean-trained neural network accurately predicts the labels.
\begin{theorem} \label{thm:good example labeld correctly}
Suppose $\cs$ satisfies the assumption \ref{ass:dist} and clean training has attained convergence under the setting of Theorem \ref{thm:kkt}. Let $\bxt \in \reals^D$ and $r \in [k]$ such that $\bxt$ is a ``nice example'' (Definition \ref{def: nice example}). Then with probability $1-2\exp(-\frac{D}{16})$, $\sign\left(\cn_{\btheta}(\bxt)\right) = y^{(r)}$.
\end{theorem}
Theorem \ref{thm:good example labeld correctly} states that the trained neural network can generalize well with a high probability when we have an example from the low-dimensional data space. 
Asymptotically, as the ambient dimension  $(D \to \infty)$ increases, this probability converges to 1.

However, even if the clean training generalizes well, our following theorem states that for any \emph{nice example} $\bxt$, there exists perturbation $\bz_{\shortparallel}$ and $\bz_{\perp}$ in the \emph{on-manifold} and \emph{off-manifold} direction which successfully misclassify the model with high probability. 
\begin{theorem}
     \label{thm:pert flips}
	Suppose $\cs$ satisfies Assumption \ref{ass:dist} and clean training has attained convergence under the setting of Theorem \ref{thm:kkt}.  
Let $\bxt \in \reals^D$ and $r \in [k]$ such that $\bxt$ is a ``nice example'' (Definition \ref{def: nice example}). 
	Then there exist  perturbations $\bz_{\perp}\propto \bu_{\perp},\,\bz_{\shortparallel}\propto \bu_{\shortparallel}$ with probability at least $1-\delta_1,1-\delta_2$ respectively; Such that,  $\cn_\btheta(\bxt - \bz_{\perp}),\, \cn_\btheta(\bxt - \bz_{\shortparallel}) \leq -1$ for $r\in Q_+$ and $\cn_\btheta(\bxt + \bz_{\perp}),\, \cn_\btheta(\bxt + \bz_{\shortparallel}) \geq 1$ for $r\in Q_-$. Also, 
 \begin{align*}
    &\norm{\bz_{\perp}}=\co{\left(\frac{d}{c\sqrt{kg\tau^2}}\right)},
    \norm{\bz_{\perp}}_{\infty}=\co{\left(\frac{d\sqrt{2\log (2g)}}{cg\tau}\right)}, \\ &\norm{\bz_{\shortparallel}}=\co{\left(\sqrt{\frac{d}{c^2k(2+\tau^2)}}\right)},\\ &\norm{\bz_{\shortparallel}}_{\infty}=\co{\left(\frac{\sqrt{d}+\tau\sqrt{2\log (2d)}}{c(2+\tau^2)}\right)}  ,
 \end{align*}
 where $ \bu_{\shortparallel}=\sum_{q\in [k]}y^{(q)}\left(\bmut^{(q)}+\bP_o\bzeta^{(q)}\right)$ and  $\bu_{\perp}=\sum_{q\in [k]}y^{(q)}\bP\bzeta^{(q)}$. $Q_+, Q_-$ are a set of clusters with labels $1,-1$ respectively and $c=k^{-1}\min\{|Q_+|,|Q_-|\}$. Probabilities $\delta_1, \delta_2$ are defined in Lemmas \ref{lem: w.h.p shifts nearly orthogonal in off manifold space}, \ref{lem: w.h.p shifts nearly orthogonal in manifold space} respectively.
\end{theorem}
Theorem \ref{thm:pert flips} provides an upper bound on the strength of an \emph{on-manifold} ($\bz_{\shortparallel}$) and \emph{off-manifold} ($\bz_{\perp}$) adversarial perturbation for any \emph{nice example}. These upper bounds hold with probability $1-\delta_1, 1-\delta_2$ for on/off-manifold perturbations respectively, where $\delta_1\to 0$ as $D\to \infty$ and $\delta_2\to 0$ as $D,d\to \infty$.\\ Note that $\tau^2=\co{(d/D)}$ based on our assumptions (Section \ref{subsec: assump}), hence we see that $\ell_2$ strength of $\bz_{\perp}$ decreases as the dimension gap $g$ increases and asymptotically (in $g$) is equivalent to $\co{(d/(c\sqrt{k}))}$. For $\bz_{\shortparallel}$ the $\ell_2$ strength increases and saturates to $\co{(\sqrt{d/(c^2k)})}$ as $g\to\infty$. 

On the other hand, the $\ell_\infty$ strength of $\bz_{\perp}$ and $\bz_{\shortparallel}$ goes to $0$ and $\co{(\sqrt{d})}$, respectively as $g\to \infty$. That is, asymptotically, the $\ell_2$ strength  of \emph{on-manifold} is a bit smaller (or comparable \footnote{When $\tau^2=\Theta(\frac{d}{D})$ then $\snorm{\bz_{\perp}}$ and $\snorm{\bz_{\shortparallel}}$ are both $\co{(\sqrt{\frac{d}{c^2k}})}$ as $g\to\infty$~.}) than the \emph{off-manifold} perturbation, but in contrast the $\ell_\infty$ strength for the \emph{off-manifold} perturbation goes to 0 while the \emph{on-manifold} perturbation is still some constant depending on $d$. Additionally, the asymptotic upper bound for \emph{on-manifold} perturbation reduces to the rate to \cite{frei2023double}; our framework is an extension introducing dimension gap and two distinct attack spaces, \cite{frei2023double} results can be related to our \emph{on-manifold} perturbation results. As these are upper-bound results to an existential result, they also act as an upper bound to the smallest adversarial perturbation. In Section \ref{sec: experiments}, we will validate the tightness of these upper bounds with empirical experiments.

\begin{figure*}[h]
    \centering
    \includegraphics[scale=0.5]{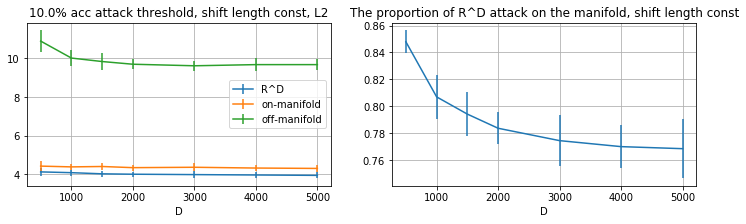}
    \includegraphics[scale=0.5]{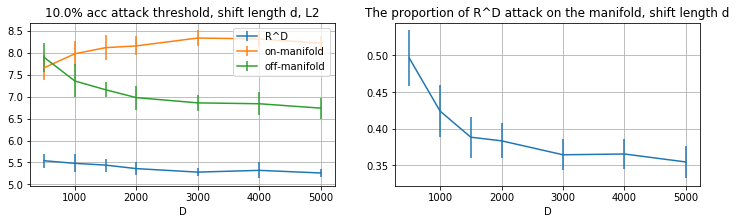}\vspace{-0.2in}
    \caption{Attack strength threshold associating with $10\%$ robust test accuracy (left) and the on-manifold proportion in $\mathbb{R}^D$ attack. The top and bottom rows correspond to the cases $\snorm{\bzeta}=\Theta(1),\Theta(\sqrt{d})$ respectively.}
    \label{fig:change_D}\vspace{-0.1in}
\end{figure*}

\paragraph{Rate comparison with \cite{melamed2023adversarial}}
\label{remark:rate_comparison} Despite the differences in \cite{melamed2023adversarial} and our theoretical setting ( Remark \ref{remark:setting}), we can attempt to draw some parallels with their main result. In \cite{melamed2023adversarial}, an adversarial perturbation (off-manifold) exists as a function of orthogonal weights when an observation $\bxt$ strictly resides in a $d$ dimensional linear subspace of a $D$ dimensional Euclidean space (dimension gap $g=D-d$). The $\ell_2$ strength of this off-manifold perturbation in our notation can be expressed as $\co{\left(N(\bxt)\sqrt{\frac{D}{g}}\right)}$ up to some constants. Here, $N(\cdot)$ represents their two-layer ReLU lazy-trained network. For the case when $\tau^2=\Theta\left(\frac{d}{D}\right)$ in  Theorem \ref{thm:pert flips}, we can obtain a similar rate for $\snorm{\bz_{\perp}}$ in our framework. However, for smaller $\tau^2$, the rate in our result is smaller. 
\section{Experiments}
\label{sec: experiments}
In this section, we conduct a simulation study and experiments in the MNIST dataset to exhibit that the data dimension plays an essential role in adversarial robustness. 
Due to the page limit, we postpone some figures to the appendix.

\subsection{Simulation Study}
\label{sec: simulation}

To generate the data, we first generate $k$ means $\{\bmu^{(r)}\}_{r=1}^k$ corresponding to the $k$ clusters in the lower $d$ dimension from $\bmu^{(r)}\sim N(0,\bI_d)$. This enforces that $\snorm{\bmu^{(r)}}=\Theta(d)$ with high probability, which is a relaxed version of assumption (A1). Given the $k$ cluster means, we generate $n=1000$ intrinsic data signals $\{\bx_i\}_{i=1}^{1000}$ centered around these means drawn uniformly. For each $i$, $\bx_i=\bmu_i+\bxi_i$, where $\bmu_i\sim \mathcal{U}(\bmu^{(1)},\ldots, \bmu^{(k)})$ and $\bxi_i\sim N(0,\bI_d/d)$, which enforces assumption (A3) with high probability.

After getting the intrinsic data signals $\{\bx_i\}_{i=1}^{1000}$, we immerse them into a $D (\gg d)$ dimension space to get our observed data samples $\{\bxt_i\}_{i=1}^{1000}$.
In alignment with our theoretical setting, we apply a matrix $\bM$ for the initial linear transformation. 
We emulate this matrix by generating $d$, $D$-dimensional  $N(0,D^{-1}\bI_D)$  Gaussian vectors $m_1,\ldots, m_d$; which act as the column of $\bM$. This is justified as when $D\gg d$ and $i\neq j$, $m_i^{\top}m_j=O_p(1/\sqrt{D})$, i.e., the $d$ vectors are almost orthogonal. 
Lastly, we add translations $\{\bzeta^{(r)}\}_{r=1}^k$ corresponding to each cluster. We generate $\bzeta^{(r)}\sim N(0,\tau^2\bI_D)$ for each cluster $r=1,\ldots,k$, the choice of $\tau^2=D^{-1},dD^{-1}$ corresponds to $\snorm{\bzeta^{(r)}}=\Theta(1),\Theta(\sqrt{d})$ respectively, with high probability. The architecture and training details can be found in Appendix \ref{appendix: sim train setting}.

To validate our theory, we need to find the minimal attack strength required to misclassify an example. This is not computationally feasible; hence as an alternative, we find the minimal attack strength required to misclassify a large proportion of the data.
We iterate over attack strengths until reaching 10\% robust testing accuracy. The details are in Appendix \ref{appendix: minimal strength attack}.


In the first simulation, we fix the intrinsic dimension $d=100$ and vary the ambient dimension $D$. The results are summarized in Figure \ref{fig:change_D}. The top panel of Figure \ref{fig:change_D} is the $\ell_2$ attack strength for $\mathbb{R}^D$, on-manifold, and off-manifold attacks. The bottom panel is the proportion of $\mathbb{R}^D$ attack that lies on the manifold. The rows of Figure \ref{fig:change_D} corresponds to the cases when $\snorm{\bzeta^{(r)}}$ is $\Theta(\sqrt{d})$ and $\Theta(1)$ or $\tau^2=dD^{-1},D^{-1}$ respectively.

There are two observations. First, the trends for $\ell_2$ attack strength in Figure \ref{fig:change_D} are as expected as in the discussion in Theorem \ref{thm:pert flips}.

Second, the bottom panel of the figure shows the ratio of the on-manifold component to the off-manifold component for the $\mathbb{R}^D$ attack. We see that as $D$ increases, this proportion decreases, implying that even when the attack has no constraint with an increase in $D$ or $g$, the model becomes more vulnerable to off-manifold perturbations.

Figure \ref{fig:simulation:inf} in the appendix is analogous to Figure \ref{fig:change_D} but considering the $\ell_\infty$ norm of the attack. The observations are consistent with the discussion succeeding Theorem \ref{thm:pert flips}. One difference is that for $\ell_{\infty}$ upper bound of the on-manifold attack, it also approaches $0$ instead of a constant (in $D,g$) $\co{(\sqrt{d})}$. This is an artifact of generating $\bM$ from random Gaussian vectors instead of the fixed-$\bM$ scenario considered in Theorem \ref{thm:pert flips}.

In the appendix, we also plot the attack threshold generated by $\ell_{\infty}$ attacks. 
The results are provided in Figure \ref{fig:simulation:inf_}. They are identical to the results of Figure \ref{fig:simulation:inf}, implying our bounds provide accurate guidance for the $\ell_{\infty}$ case. Intuitively, when $D$ increases, the total attack budget grows when fixing the $\ell_{\infty}$ attack strength, so we observe a decreasing trend in the top panel of Figure \ref{fig:simulation:inf_}. Besides, if we multiply the attack strength by $\sqrt{D}$, we can observe a similar trend as the $\ell_2$ attack in Figure \ref{fig:change_D}, consistent with our prior discussions.


In contrast to the first simulation, we fix the ambient dimension $D=2000$ in the second simulation and vary the intrinsic dimension $d$. We track the attack threshold as a function of $d$. The results are provided in the appendix, Figure \ref{fig:change_codim}. One can see that when $d$ increases, there is more on-manifold information; thus, the model is more robust.
\paragraph{Effect of Initialization}
\label{sec: init}
We also conduct experiments to examine the effect of neural network initialization and $\ell_2$ regularization (weight decay). In the appendix, the results are postponed to Figure \ref{fig:change_init} and \ref{fig:change_decay}. The trained model slightly improves when using a small initialization or weight decay compared to Figure \ref{fig:change_D}. However, the vulnerability towards off-manifold attack in these two scenarios is similar to Figure \ref{fig:change_D}. That is, dimension-gap hampers robustness regardless of the initialization.
\subsection{MNIST Dataset}
\label{subsec: mnist_exp}
We also conduct experiments on the MNIST dataset to validate the effect of the gap between intrinsic and ambient dimensions on the adversarial robustness. According to our theoretical setup, the MNIST dataset can be considered as a union of low-dimensional local clusters immersed in a high-dimensional space as bolstered by Table \ref{tab:mnistclasses} in the appendix.

Different from the simulation, we cannot access the low-dimensional intrinsic signal corresponding to the MNIST images; the original data set already has some dimension gap. Regardless, we can still attempt to track the effect of the additional dimension gap on the adversarial robustness. To make a fair assessment, we implement \emph{isometric transformations} on the original MNIST images that increase the ambient dimension $D$ to track the dependence of attack strength with varying dimensions.

\begin{figure}[h]\vspace{-0.2in}
\includegraphics[scale=0.3]{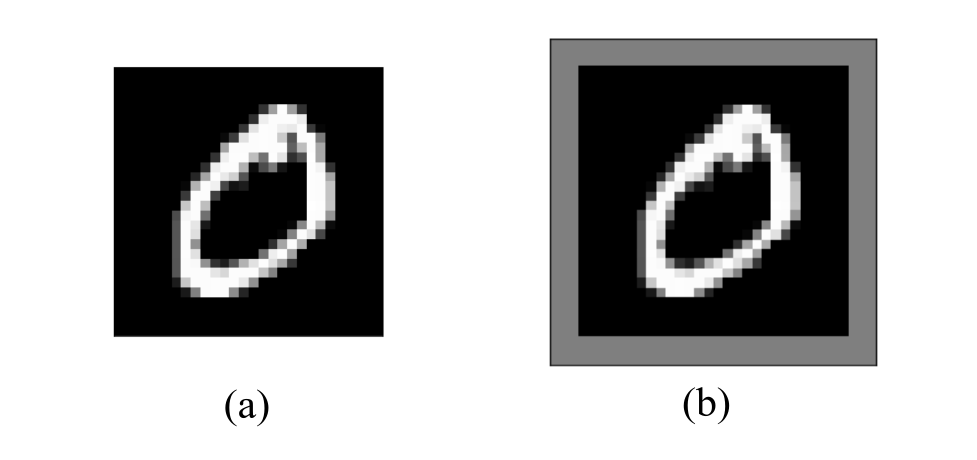}\vspace{-0.2in}
\caption{(a) Original image 28x28, $D=784$; (b) Padded image 34x34, $P=3$, $D=1156$.}\vspace{-0.1in}
\label{fig:padding}
\end{figure}
We pad gray (0.5) pixels on the image's border to increase the ambient dimension (Figure \ref{fig:padding}). Adding a pad of $P$ pixels along the borders increases the ambient dimension $D$ by $4P^2+2P$. Adding gray borders to the images is an isometric transformation, i.e., the distances between every image are preserved, and no information is lost. Padding the images will allow us to emulate the effect of increasing the dimension gap without distorting the data. Hence the attack strengths across dimensions are still comparable.

We use a convolution network to train the model. The architecture and training details are in Appendix \ref{appendix: mnist setting}.

 We consider two settings: (1) training for 10 epochs and (2) training until the loss is 0.02. While presenting the results for (1) in the main content, we postpone the results for (2) in the appendix. After training the neural network, we examine the robustness against $\ell_2$ and $\ell_\infty$ attacks in the testing data, and the results are summarized in Figure \ref{fig:mnist_10epoch_l2} and \ref{fig:mnist_10epoch_linf}.

From our theory, it is expected that when increasing the ambient dimension $D$, the model becomes less robust, i.e., we require smaller $\ell_2$ attacks until the norm saturates to some constant. In the top panel of Figure \ref{fig:mnist_10epoch_l2}, one can see that when $P$ increases (in turn increasing $D$), we only need a smaller $\ell_2$ attack to achieve a low accuracy and this small value converges to some constant which aligns with the theory. In the bottom panel of Figure \ref{fig:mnist_10epoch_l2}, we plot the accuracy w.r.t. different attack strengths.
For $\ell_\infty$ attack, based on the observations for $\ell_2$ attack in Figure \ref{fig:mnist_10epoch_l2}, it is intuitive that the neural network is less robust against $\ell_{\infty}$ attack as well when $D$ increases as observed in Figure \ref{fig:mnist_10epoch_linf}.

\begin{figure}[!ht]
    \centering\vspace{-0.05in}
    \includegraphics[scale=0.45]{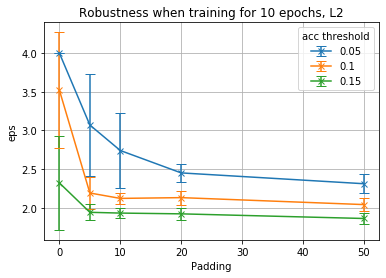}
    \includegraphics[scale=0.45]{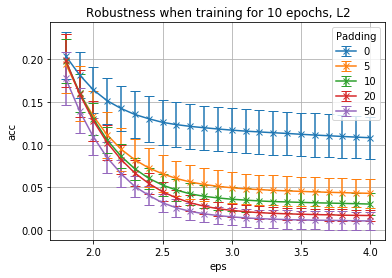}\vspace{-0.2in}
    \caption{The relationship between attack strength and the padding number $P$ in Table \ref{tab:my_label} (Top) and the relationship between robust test accuracy and attack strength for different values of $P$ (Bottom). $\ell_2$ attack. 10 training epochs.} 
    \label{fig:mnist_10epoch_l2}
\end{figure}

\begin{figure}[!ht]
    \centering
    \includegraphics[scale=0.45]{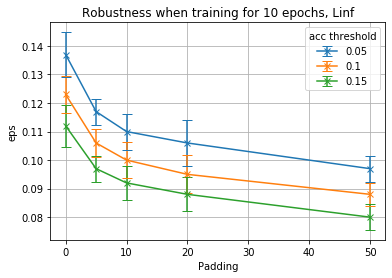}
    \includegraphics[scale=0.45]{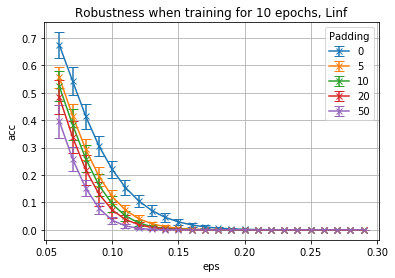}\vspace{-0.2in}
    \caption{The relationship between attack strength and the padding number $P$ (Top), and the relationship between robust test accuracy and attack strength under different $P$.(Bottom) $\ell_\infty$ attack. 10 training epochs.} 
    \label{fig:mnist_10epoch_linf}
\end{figure}
\subsection{Fashion-MNIST}
In the same vein as Section \ref{subsec: mnist_exp}, we verify the effect of increased dimension gap (by padding grey pixel as in the MNIST experiment) on the vulnerability of the model on FMNIST images \citep{xiao2017/online} in Figure \ref{fig:fmnist_30epoch_linf}.
Figure \ref{fig:fmnist_30epoch_linf} illustrates the same pattern of the bottom plot in Figure \ref{fig:mnist_10epoch_linf}.
We use the same convolution network architecture and training setting used for the MNIST experiment (refer to Section \ref{appendix: mnist setting}) averaging over 30 random seeds.
\begin{figure}[!ht]
    \centering
    \includegraphics[scale=0.5]{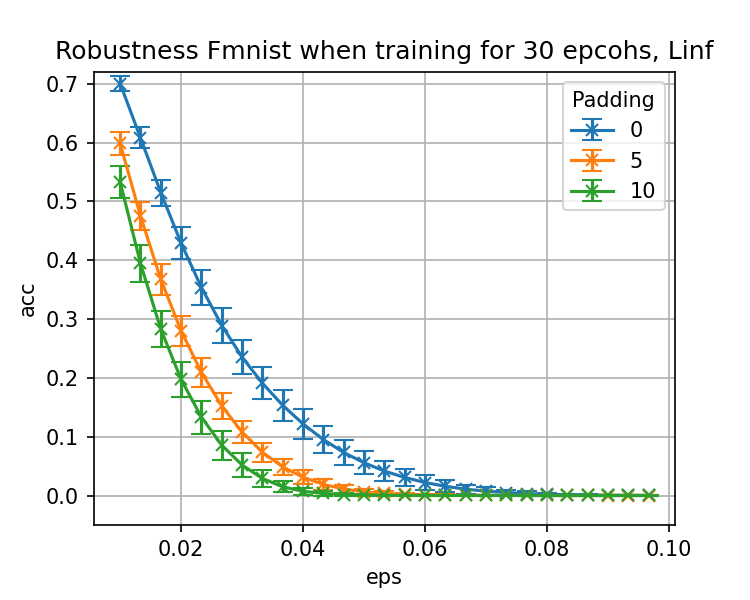}\vspace{-0.2in}
    \caption{Relationship between robust test accuracy and attack strength under different $P$. $\ell_\infty$ attack. 30 training epochs.} 
    \label{fig:fmnist_30epoch_linf}
\end{figure}
\subsection{Imagenet Binary Classification}
We subset two classes (Dog and Church) from the Imagenet dataset \citep{5206848} and scale all the images to a uniform resolution of 320x320 pixels. In contrast to the FMNIST/MNIST experiments, where we emulate the effect of increasing the dimension gap via adding padding to the images, for Imagenet, we scale the images to different resolutions. The motivation of this experiment is that for higher-resolution images, the ambient dimension increases disproportionately w.r.t the useful features or on-manifold dimensions. From Table \ref{tab: intrinsic_dimension_resolution}, even with different resolutions, the intrinsic dimension of the data remains almost constant. Hence, we can use resolution as a surrogate for the intrinsic-ambient dimension gap: higher-resolution images will tend to have a higher dimension gap. \\

\begin{table}[h]
\begin{tabular}{@{}lllll@{}}
\toprule
Size($\text{px}^2$) & Ambient & LPCA & MLE(k=5) & TwoNN \\ \midrule
320         & 307200  & 12   & 22.75    & 31.27 \\ 
128        & 49152   & 12   & 21.87    & 30.08 \\ 
64       & 12288   & 12   & 20.88    & 28.93 \\ \bottomrule
\end{tabular}
\caption{Intrinsic dimension is unchanged even with varying resolution. Imagenet (Dogs+Church).}
\vspace{-1.8mm}
\label{tab: intrinsic_dimension_resolution}
\end{table}
We rescale the images to 64x64, 128x128, and 320x320 (base) pixels and train Resnet18 models \citep{He_2016_CVPR} with $10^{-3}$ learning rate and Adam optimizer for 20 epochs in each resolution setting. All models have been trained for 20 epochs for uniform comparison, with $\sim 0.02$ training loss and $~99\%$ accuracy averaging over 30 seeds. After the clean training, the models are subjected to $\ell_\infty$ PGD attack of varying strengths up to $\epsilon=16/255$.

Figure \ref{fig:imagenet_linf} illustrates that, as the resolution of an image increases, the vulnerability also increases, bolstering our dimension gap argument further.
\begin{figure}[!ht]
    \centering
    \includegraphics[scale=0.5]{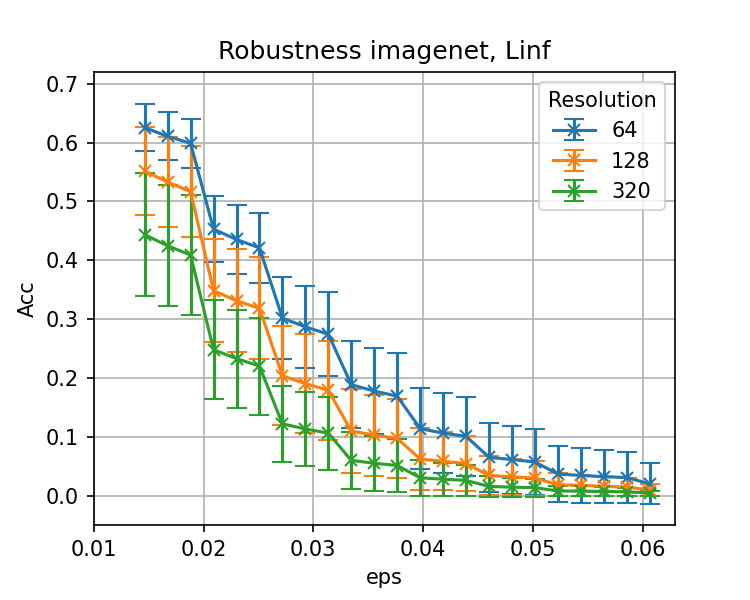}\vspace{-0.2in}
    \caption{Relationship between robust test accuracy and attack strength under different resolutions. $\ell_\infty$ attack.} 
    \label{fig:imagenet_linf}
\end{figure}
\section{Discussion and Conclusion}
We introduce the notion of natural and unnatural attacks residing in on and off manifold space concerning the low dimensional data space. Despite the generalization capabilities of a cleanly trained model on examples drawn from the data space, it will exhibit vulnerabilities to perturbations in both off and on manifold directions. In particular, the off-manifold adversarial perturbations can get arbitrarily small (in $\ell_{\infty}$ norm) with the increase in dimension gap. 
\paragraph{Adversarial Training} This work discusses the intrinsic-ambient dimension gap as the root cause for prevalent adversarial attacks in literature. However, we have yet to comment on adversarially trained models that are used to induce robustness to such attacks. We attempt to motivate the success of adversarial training through the lens of our framework. In contrast to clean training, adversarial training aims to minimize the loss over the worst possible attacks the data can exhibit. We can think of adversarial training as clean training on a dataset consisting of adversarial examples of the original data. Even though the original data space is a low dimensional subset of the whole Euclidean space, the dimension of adversarial data space will always be equal to the ambient dimension as we consider a $D$ dimensional $\epsilon$ ball around the original data space. Where $\epsilon$ is a parameter determining the strength of attacks, adversarial data space makes the dimension gap $0$ and, based on our theory, should not induce any \emph{off-manifold} or \emph{unnatural attacks}. Thinking of adversarial training as clean training on a much higher dimensional data space also sheds light on the complexity of adversarial training \citep{schmidt2018adversarially}. A higher dimensional data space requires many more observations for a representative sample.
\newpage
\bibliographystyle{asa}
\bibliography{reference}

\section*{Checklist}
 \begin{enumerate}

 \item For all models and algorithms presented, check if you include:
 \begin{enumerate}
   \item A clear description of the mathematical setting, assumptions, algorithm, and/or model. \\\textbf{Yes}, Refer to Sections \ref{subsec: data}, \ref{subsec: assump}, \ref{sec: experiments} respectively.
   \item An analysis of the properties and complexity (time, space, sample size) of any algorithm. \\\textbf{Not Applicable}
   \item (Optional) Anonymized source code, with specification of all dependencies, including external libraries. \\\textbf{No}, we only have simulations and simple experiments for which the details provided in Section \ref{sec: experiments}, Appendix \ref{Additional details exp} suffice.
 \end{enumerate}

 \item For any theoretical claim, check if you include:
 \begin{enumerate}
   \item Statements of the full set of assumptions of all theoretical results. \\\textbf{Yes}, Refer section \ref{subsec: assump}.
   \item Complete proofs of all theoretical results. \\\textbf{Yes}, Refer to the proofs in Appendix \ref{proofs: apendix}.
   \item Clear explanations of any assumptions. \\\textbf{Yes}, Refer section \ref{subsec: assump}. 
 \end{enumerate}

 \item For all figures and tables that present empirical results, check if you include:
 \begin{enumerate}
   \item The code, data, and instructions needed to reproduce the main experimental results (either in the supplemental material or as a URL). \\\textbf{No}, We are using simple simulations.
   \item All the training details (e.g., data splits, hyperparameters, how they were chosen).\\ \textbf{Yes}, details provided in Section \ref{sec: experiments}, Appendix \ref{Additional details exp}.
         \item A clear definition of the specific measure or statistics and error bars (e.g., with respect to the random seed after running experiments multiple times). 
        \\
        \textbf{Yes}, We provide error bars by running multiple runs. 
         \item A description of the computing infrastructure used. (e.g., type of GPUs, internal cluster, or cloud provider). \textbf{No}, Computation is inexpensive and doesn't need specialized infrastructure.
 \end{enumerate}

 \item If you are using existing assets (e.g., code, data, models) or curating/releasing new assets, check if you include:
 \begin{enumerate}
   \item Citations of the creator If your work uses existing assets.\\ \textbf{Yes}, Only external asset used is MNIST dataset which has been cited.
   \item The license information of the assets, if applicable.\\ \textbf{Not Applicable}
   \item New assets either in the supplemental material or as a URL, if applicable. \\\textbf{Not Applicable}
   \item Information about consent from data providers/curators.\\ \textbf{Not Applicable}
   \item Discussion of sensible content if applicable, e.g., personally identifiable information or offensive content. \textbf{Not Applicable}
 \end{enumerate}

 \item If you used crowdsourcing or conducted research with human subjects, check if you include:
 \begin{enumerate}
   \item The full text of instructions given to participants and screenshots. \\\textbf{Not Applicable}
   \item Descriptions of potential participant risks, with links to Institutional Review Board (IRB) approvals if applicable. \\\textbf{Not Applicable}
   \item The estimated hourly wage paid to participants and the total amount spent on participant compensation. \\\textbf{Not Applicable}
 \end{enumerate}

 \end{enumerate}
\appendix
\onecolumn
\section{Proof Sketch}
We adapt the mathematical techniques used in \cite{frei2023double} to our framework.
\begin{itemize}
    \item[1.] We use the KKT condition in \eqref{eqn: main kkt} to connect the neural network weights $\bw_j,b_j$ to the data $\bxt_i,y_i$. Then we can write the outputs at a new sample $\bxt$  (i.e., $\cn_{\btheta}(\bxt)=\sum_{j=1}^wv_j\phi(\bw_j^\top\bxt+b_j)$) in terms of interactions ($\inner{\bxt_i,\bxt}$) between the observed signal $\bxt_i$ and $\bxt$. For the values of $\inner{\bxt_i,\bxt}$, the means $\bmu^{(r)}$, $r\in[k]$  are nearly orthogonal [\ref{ass:dist} (A2)] and impose $\inner{\bxt_i,\bxt}$ to be larger, if $\bxt_i,\bxt$ belong to the same cluster compared to the case when they are from different clusters. This ensures that the neuron output is consistent when a sample is from the same cluster. (Theorem \ref{thm:good example labeld correctly})
    \item[2.] We further study the adversarial attack. As discussed in Section \ref{subsec: volatile bias}, $\bw_j$ can be decomposed into $\bP_o\bw_j$ and $\bP\bw_j$. We design perturbations $\bz_{\shortparallel}$ and $\bz_{\perp}$, such that $\bP_o\bw_j^\top\bz_{\shortparallel}$ and $\bP\bw_j^\top\bz_{\perp}$ can be made arbitrarily large in magnitude compared to $\bw_j^\top\bxt$. For $\bz=\bz_{\perp},\bz_{\shortparallel}$, we have $\bw_j^{\top}(\bxt+\bz)=\bw_j^{\top}\bxt+\bP\bw_j^{\top}\bz_{\perp},\bw_j^{\top}\bxt+\bP_o\bw_j^{\top}\bz_{\shortparallel}$; thus we can manipulate the sign of each neuron after perturbation. (Theorem \ref{thm:pert flips})
\end{itemize}

\section{Definitions and short hands}
$\Bar{\zeta}$ is a shorthand notation to bound $\bzeta^{(r)}+\bomega_i$ from above, for any $i\in[n],r\in [k]$ with very high probability. $\Bar{\omega}$ will denote the upper bound for $\bomega_i$ for $i\in [n]$. If assumption \ref{ass:dist} hold then $\Bar{\omega}\leq 1$ and $\Bar{\zeta}$ can be set to a  constant in Definition \ref{def:constants}.
\begin{definition}[Semi-inner product]
\label{def:semi_inner}
For any positive semi-definite matrix $A$ we can define a semi-inner product $\inner{,}_{A}$ such that for all $a,b\in\mathbb{R}^D$:
\begin{equation*}
    \inner{a,b}_A=a^TAb
\end{equation*}
If $A$ is a projection matrix, then $\inner{a,b}_A=\inner{Aa,Ab}$ where $\inner{,}$ is the standard inner product.
\end{definition}
\begin{definition}[Constants]
\label{def:constants}
\begin{align*}
    c_2&=1+\sqrt{2}; \Bar{\zeta}^2=1+2\tau\sqrt{\frac{13D}{8}}+\frac{13D\tau^2}{8}\\
    \Delta &= 2c_2 \sqrt{d} \ln(d)\\
    \ex&=\Bar{\zeta}^2+2\Bar{\zeta}\sqrt{d}c_2\\
    \exi&=\Bar{\zeta}^2+2\Bar{\zeta}\sqrt{2}\ln{d}\\
    \exxi&=\Bar{\zeta}^2+\Bar{\zeta}(2\sqrt{2}\ln{d}+\sqrt{d}c_2)\\
    \epsilon_{\zeta}^{\shortparallel}&=\left(\sqrt{2}\ln d+\Bar{\omega}\right)\norm{\bP_o\bzeta};\, \epsilon_{\zeta}^{\perp}=\Bar{\omega}\norm{\bP\bzeta}\\
     p&=\max_{i \neq j} |\inner{\bmu^{(i)}, \bmu^{(j)}}|\\
     \Delta'&=\Delta+\ex
\end{align*}
\end{definition}

\section{Proofs}
\label{proofs: apendix}
\begin{lemma}[Properties]
    \label{lem:properties}
    Given $\bx$ belongs to the $r^{th}$ cluster, if
   the assumptions \ref{ass:dist} hold and $\norm{\bzeta^{(r)}+\bomega}\leq \Bar{\zeta}$ for any $r\in [k]$, we have the following properties:
    \begin{enumerate}
       \item \label{c1}$| \inner{\bxi_i,\bxi}| < \Delta$ and $ |\inner{\bxit_i,\bxit}| \leq \Delta+\exi\leq \Delta'$.
		\item \label{c2}$| \inner{\bx_i,\bxi}| \leq \Delta$ and $| \inner{\bxt_i,\bxit}| \leq \Delta+\exxi\leq \Delta'$
		\item \label{c3}If $i \not \in I^{(r)}$ then $|\inner{\bx_i,\bx}| \leq p + \Delta$ and $|\inner{\bxt_i,\bxt}| \leq p + \Delta'$
		\item \label{c4}If $i \in I^{(r)}$ then $| \inner{\bx_i,\bx} - d | \leq \Delta$ and $| \inner{\bxt_i,\bxt} - d | \leq \Delta'$
    \end{enumerate}
\end{lemma}
\begin{proof}
Using cauchy-schwartz, (A3) and $\ln x < \sqrt{x}$; we have:
    \begin{enumerate}
        \item   $| \inner{\bxi_i,\bxi}|\leq |2\ln^2 d|<2\sqrt{d}\ln d<\Delta$. Consequently, \begin{align*}
            |\inner{\bxit_i,\bxit}|&=|\inner{\bM\bxi_i+\bzeta^{(r')}+\bomega_i,\bM\bxi+\bzeta^{(r)}+\bomega}|\\
&=|\inner{\bxi_i,\bxi}+\inner{\bzeta^{(r')}+\bomega_i,\bzeta^{(r)}+\bomega}+\inner{\bM\bxi_i,\bzeta^{(r)}+\bomega}+\inner{\bM\bxi,\bzeta^{(r')}+\bomega_i}|\\
&\leq |\inner{\bxi_i,\bxi}| +\Bar{\zeta}^2+2\snorm{\bM}_{op}\sqrt{2}\ln d \cdot\Bar{\zeta}  \leq  \Delta+\epsilon_{\xi}\leq \Delta+\epsilon_{x}=\Delta'
\end{align*}
$\snorm{\bM}_{op}=1$ as $\bM$ has orthogonal columns.
\item $| \inner{\bx_i,\bxi}|= |\inner{\bmu^{(r')}+\bxi_i,\bxi}|\leq |\inner{\bmu^{(r')},\bxi}|+|\inner{\bxi_i,\bxi}|\leq \sqrt{2}\sqrt{d}\ln d+2\ln^2d<2c_2\sqrt{d}\ln d=\Delta$.
 \begin{align*}
      \inner{\bxt_i,\bxit}|&=|\inner{\bM\bx_i+\bzeta^{(r')}+\bomega_i,\bM\bxi+\bzeta^{(r)}+\bomega}|\\
      &\leq |\inner{\bx_i,\bxi}+\inner{\bzeta^{(r')}+\bomega_i,\bzeta^{(r)}+\bomega}+\inner{\bM\bxi_i,\bzeta^{(r)}+\bomega}+\inner{\bM\underbrace{\bx_i}_{\bmu^{(r')}+\bxi_i},\bzeta^{(r')}+\bomega_i}|\\
      &\leq \Delta+\Bar{\zeta}^2+\snorm{\bM}_{op}\sqrt{2}\ln d\cdot\Bar{\zeta}+\snorm{\bM}_{op}(\sqrt{d}+\sqrt{2}\ln d)\cdot\Bar{\zeta}\\
      &\leq\Delta+\exxi\leq \Delta+\ex=\Delta'
  \end{align*}
  \item If $i \not \in I^{(r)}$ then, $|\inner{\bx_i,\bx}|=|\inner{\bmu^{(r')}+\bxi_i,\bmu^{(r)}+\bxi}|\leq |\inner{\bmu^{(r')},\bmu^{(r)}}|+|\inner{\bxi_i,\bxi}|\leq p+\Delta$.\\
  Similarly, $|\inner{\bxt_i,\bxt}|=|\inner{\bmut^{(r')}+\bxit_i,\bmut^{(r)}+\bxit}|\leq |\inner{\bmut^{(r')},\bmut^{(r)}}|+|\inner{\bxit_i,\bxit}|=|\inner{\bmu^{(r')},\bmu^{(r)}}|+|\inner{\bxit_i,\bxit}|\leq p+\Delta'$.
  \item If $i  \in I^{(r)}$ then, $|\inner{\bx_i,\bx}-d|=|\inner{\bmu^{(r)}+\bxi_i,\bmu^{(r)}+\bxi}-d|\leq \snorm{\bmu^{(r)}}-d+\inner{\bxi_i,\bxi}|\leq 0+\Delta$.\\
  Similarly, $|\inner{\bxt_i,\bxt}-d|=|\inner{\bmut^{(r)}+\bxit_i,\bmut^{(r)}+\bxit}-d|\leq \snorm{\bmu^{(r)}}-d+\inner{\bxit_i,\bxit}|\leq 0+\Delta'$.
    \end{enumerate}
\end{proof}
\subsection{High probability statements}
In this section, we show that any of the following events hold with a very high probability for any $q\in[k]$ :
\begin{enumerate}[label=(E\arabic*)]
    \item $\norm{\bzeta^{(q)}}\asymp \tau\sqrt{D}$, in particular $\tau\sqrt{\frac{D}{2}}<\norm{\bzeta^{(q)}}<\tau\sqrt{\frac{13D}{8}}$
    \item $\bzeta^{(q)\top}\bP\bzeta^{(q)}\asymp g\tau^2$, in particular $\frac{g\tau^2}{2}<\bzeta^{(q)\top}\bP\bzeta^{(q)}<\frac{13g\tau^2}{8}$
    \item $\bzeta^{(q)\top}\bP_{o}\bzeta^{(q)}\asymp d\tau^2$, in particular $\frac{d\tau^2}{2}<\bzeta^{(q)\top}\bP_o\bzeta^{(q)}<\frac{13d\tau^2}{8}$
    \item $\norm{\bP\bzeta^{(q)}}_{\infty}=\co{\left(\tau\sqrt{2\log(2g)}\right)}$, in particular $\norm{\bP\bzeta^{(q)}}_{\infty}<3\tau\sqrt{2\log(2g)}$
    \item $\norm{\bP_o\bzeta^{(q)}}_{\infty}=\co{\left(\tau\sqrt{2\log(2d)}\right)}$, in particular $\norm{\bP_o\bzeta^{(q)}}_{\infty}<3\tau\sqrt{2\log(2d)}$
    \item $k\cdot\left(\max_{i\neq j}|\inner{\bzeta^{(i)},\bzeta^{(j)}}_{\bP}|+\epsilon^{\perp}_{\zeta}\right)<\frac{1}{10}\cdot \left(\bzeta^{(q)T}\bP\bzeta^{(q)}-\epsilon^{\perp}_{\zeta}\right)$
    \item $k\cdot\left(\max_{i\neq j}|\inner{\bzeta^{(i)},\bzeta^{(j)}}_{\bP_o}|+\epsilon^{\shortparallel}_{\zeta}\right)<\frac{1}{10}\cdot \left(\bzeta^{(q)T}\bP_o\bzeta^{(q)}-\epsilon^{\shortparallel}_{\zeta}\right)$
\end{enumerate}
\begin{lemma}
\label{lem:norm/quadratic form bounded whp}
    $\norm{\bzeta^{(q)}}\asymp \tau\sqrt{D}=\sqrt{h}$ , $\bzeta^{(q)\top}\bP\bzeta^{(q)}\asymp g\tau^2$ and $\bzeta^{(q)\top}\bP_o\bzeta^{(q)}\asymp d\tau^2$ w.h.p. for any $q \in [k]$. \\Specifically, \begin{align*}
        P(\frac{\bzeta^{(q)\top}\bP\bzeta^{(q)}}{\tau^2}>\frac{13g}{8})\leq e^{-\frac{g}{16}}; \, P(\frac{\bzeta^{(q)\top}\bP\bzeta^{(q)}}{\tau^2}<\frac{g}{2})\leq e^{-\frac{g}{16}}\\
        P(\frac{\bzeta^{(q)\top}\bP_o\bzeta^{(q)}}{\tau^2}>\frac{13d}{8})\leq e^{-\frac{d}{16}}; \, P(\frac{\bzeta^{(q)\top}\bP_o\bzeta^{(q)}}{\tau^2}<\frac{d}{2})\leq e^{-\frac{d}{16}}\\
        P(\norm{\bzeta^{(q)}}>\tau\sqrt{\frac{13D}{8}})\leq e^{-\frac{D}{16}}; \, P(\norm{\bzeta^{(q)}}<\tau\sqrt{\frac{D}{2}})\leq e^{-\frac{D}{16}}
    \end{align*}

\end{lemma}
\begin{proof}
    As $\bP,\bP_o$ are projection matrices of rank $g,d$ respectively, they are idempotent and hence $\bzeta^{(q)\top}\bP\bzeta^{(q)}\tau^{-2}\sim \chi^2(g)\, ,\bzeta^{(q)\top}\bP_o\bzeta^{(q)}\tau^{-2}\sim \chi^2(d)$. Also, $\norm{\bzeta}^2\tau^{-2}\sim \chi^2(D)$.\\
Using \cite{laurent2000adaptive} lemma 1:
$$
    P(\frac{\bzeta^{(q)\top}\bP\bzeta^{(q)}}{\tau^2}>g+2\sqrt{g}\sqrt{x}+2x)\leq e^{-x}; \, P(\frac{\bzeta^{(q)\top}\bP\bzeta^{(q)}}{\tau^2}<g-2\sqrt{g}\sqrt{x})\leq e^{-x}$$
    Setting $x=\frac{g}{16}$ we get:
    $$
    P(\frac{\bzeta^{(q)\top}\bP\bzeta^{(q)}}{\tau^2}>\frac{13g}{8})\leq e^{-\frac{g}{16}}; \, P(\frac{\bzeta^{(q)\top}\bP\bzeta^{(q)}}{\tau^2}<\frac{g}{2})\leq e^{-\frac{g}{16}}$$
Using union bound we have $P(\frac{g}{2}<\frac{\bzeta^{(q)\top}\bP\bzeta^{(q)}}{\tau^2}<\frac{13g}{8})\geq1-2e^{-\frac{g}{16}}$. One can get an identical statement for $\bzeta^{(q)\top}\bP_o\bzeta^{(q)}$ replacing $g$ with $d$.\\
Similarly, one can get $$
    P(\frac{\norm{\bzeta^{(q)}}^2}{\tau^2}>\frac{13D}{8})\leq e^{-\frac{D}{16}}; \, P(\frac{\norm{\bzeta^{(q)}}^2}{\tau^2}<\frac{D}{2})\leq e^{-\frac{D}{16}}$$
    Consequently, $P(\frac{D}{2}<\frac{\norm{\bzeta^{(q)}}^2}{\tau^2}<\frac{13D}{8})\geq 1-2e^{-\frac{D}{16}}$
\end{proof}
\begin{lemma}
\label{lem:leverage <1}
    For the projection matrix $\bP$ for every $i\in[D]$ we have $0<p_{ii}<1$. Where $p_{ij}$ is the entry in $i^{th}$ row and $j^{th}$ column of $\bP$. Also, there exists an indexing set $L$ of length $d$ such that for all $i\in L$, $p_{ii}=0$.
\end{lemma}
\begin{proof}
    As $\bP^{\top}=\bP$ and $\bP^2=\bP$, we have $p_{ii}=p_{ii}^2+ \sum_{i\neq j}p_{ij}^2$.\\
    This means $p_{ii}^2\leq p_{ii}$, hence $0<p_{ii}<1$. \\
    Considering, the eigenvalue decomposition of $\bP=O^T\Lambda O$ WLOG we can assume that $\Lambda$ is ordered such that the first $g$ diagonal elements are $1$ and the rest $d$ diagonal elements are $0$. One can check that this enforces the first $g$ diagonal elements of $\bP$ to non-zero and the rest $d$ diagonal elements to be $0$.\\
    In general $\Lambda$ won't be ordered hence an indexing set $L$ can be used corresponding to the $0$ eigenvalues; and $p_{ii}=0$ for $i \in L$.
\end{proof}
\begin{corollary}
    For the projection matrix $\bP_o$ for every $i\in[D]$ we have $0<r_{ii}<1$. Where $r_{ij}$ is the entry in $i^{th}$ row and $j^{th}$ column of $\bP_o$. Also, there exists an indexing set $L_o$ of length $g$ such that for all $i\in L_o$, $r_{ii}=0$. 
    \label{cor:leverage<1}
\end{corollary}
\begin{proof}
Notice that,
    $\bP_o=\bI-\bP=O^\top(\bI-\Lambda)O$. Mimicking the arguments of \lemref{lem:leverage <1} the proof follows.
\end{proof}
\begin{lemma}
    \label{lem:linf of perturbation vectors bounded whp}
    $\norm{\bP\bzeta^{(q)}}_{\infty}=\co{\left(\tau\sqrt{2\log(2g)}\right)}\, , \norm{\bP_o\bzeta^{(q)}}_{\infty}=\co{\left(\tau\sqrt{2\log(2d)}\right)}$ w.h.p. for any $q\in [k]$.\\
    Specifically, 
    \begin{align*}
P\left[\snorm{\bP\bzeta^{(q)}}_{\infty}>3\tau\sqrt{2\log(2g)}\right]\leq e^{-4\log(2g)} \\
P\left[\snorm{\bP_o\bzeta^{(q)}}_{\infty}>3\tau\sqrt{2\log(2d)}\right]\leq e^{-4\log(2d)}
    \end{align*}
\end{lemma}
\begin{proof}

    First we will bound $\E (\snorm{\bP\bzeta}_{\infty})$ for any $\bzeta\sim N(0,\tau^2\bI_{D})$.\\
    Let, $Z=\snorm{\bP\bzeta}_{\infty}$ equivalently $Z=\max\limits_{i\in[D]}\max \{(\bP\bzeta)_i,-(\bP\bzeta)_i\}$.\\
    We can define $X_i=(\bP\bzeta)_i \sim N(0,\tau^2p_{ii})$ and $X_{2i}=(-\bP\bzeta)_i \sim N(0,\tau^2p_{ii})$. Where $p_{ii}$ is the $i^{th}$ diagonal element of $\bP$. From  \lemref{lem:leverage <1} there exists an indexing set $L$ such that $p_{ii}=0 \, \forall i\in L$. This implies $X_i=0,X_{2i}=0 \, \forall i\in L$. Hence, $Z=\max\limits_{i\in[D]\setminus L}\max \{(\bP\bzeta)_i,-(\bP\bzeta)_i\}$.\\
    By Jensen's inequality for any $t>0$, we have
    \begin{align*}
        \exp \{t\mathbb{E}[ Z] \} \leq \mathbb{E} \exp \{tZ\} = \mathbb{E} \max\limits_{i\in [2D]\setminus (L\cup2L)} \exp \{tX_i\} &\leq \sum_{i\in [2D]\setminus (L\cup2L)} \mathbb{E} [\exp \{tX_i\}] \\&\overset{(i)}= \sum_{i \in[D]\setminus L}2\exp \{t^2 p_{ii}\tau^2/2\}\overset{(ii)}\leq 2g\exp\{t^2\tau^2/2\}
    \end{align*}
    $(i)$ uses the moment-generating function for Gaussian. $(ii)$ uses the fact that $p_{ii}<1$ and $|[D]\setminus L|=g$ from \lemref{lem:leverage <1}.\\
$$\E[Z] \leq \frac{\log 2g}{t} + \frac{t \tau^2}{2}$$
Setting $t = \frac{\sqrt{2 \log 2g}}{\tau}$ we get $\E[Z]\leq \tau \sqrt{2\log(2g)}$.
For any $u>0$ we have
\begin{align*}
    P[Z>\tau \sqrt{2\log(2g)}+u]&\leq P(Z>\E Z+u)\\
    &\overset{(iii)}\leq e^{-\frac{u^2}{2\tau^2\max_i p_{ii}}}\leq e^{-\frac{u^2}{2\tau^2}}
\end{align*}
Where, $(iii)$ uses Borell-TIS inequality \cite{adler2007gaussian}.\\
Setting, $u=2\tau\sqrt{2\log(2g)}$ we get $P(Z>3\tau\sqrt{2\log(2g)})\leq e^{-4\log(2g)}$. Finally 
$P(\snorm{\bP\bzeta}_{\infty}<3\tau\sqrt{2\log(2g)})=1-e^{-4\log(2g)}$.
One can get an identical result for $\bP_o\bzeta^{(q)}$ replacing $g\to d$ and using Corollary $\ref{cor:leverage<1}$ instead of \ref{lem:leverage <1}.
\end{proof}

\begin{lemma}
    \label{lem: semi-inner product close to zero}
    For any $\delta>0$ and $i\neq j \in [k]$;  $|\inner{\bzeta^{(i)},\bzeta^{(j)}}_{\bP}|,|\inner{\bzeta^{(i)},\bzeta^{(j)}}_{\bP_o}|\leq \delta $ w.h.p.\\ Specifically, 
    \begin{align*}
        P \left[\inner{\bzeta^{(i)},\bzeta^{(j)}}_{\bP} \geq \delta \right] \leq 4e^{-\frac{2\delta}{\sqrt{8g}\tau^2}} &; P \left[\inner{\bzeta^{(i)},\bzeta^{(j)}}_{\bP} \leq -\delta \right] \leq 4e^{-\frac{2\delta}{\sqrt{8g}\tau^2}}\\
        P \left[\inner{\bzeta^{(i)},\bzeta^{(j)}}_{\bP_o} \geq \delta \right] \leq 4e^{-\frac{2\delta}{\sqrt{8d}\tau^2}} &; P \left[\inner{\bzeta^{(i)},\bzeta^{(j)}}_{\bP_o} \leq -\delta \right] \leq 4e^{-\frac{2\delta}{\sqrt{8d}\tau^2}}
    \end{align*}

\end{lemma}
\begin{proof}
    Let, $Z_1=\bP\bzeta^{(i)}\sim N(0,\tau^2\bP)$ and $Z_2=\bzeta^{(j)}\sim N(0,\tau^2\bI_{D})$ for $i\neq j$. $$\inner{\bzeta^{(i)},\bzeta^{(j)}}_{\bP}=\bzeta^{(i)\top} \bP \bzeta^{(j)}=(\bP \bzeta^{(i)})^\top \bzeta^{(j)}=Z_1^TZ_2$$
    The inner product can be expressed as:
    $Z_1^TZ_2=\snorm{Z_1}\cdot\frac{Z_1}{\snorm{Z_1}}^\top Z_2$.\\
    Notice that, $\frac{Z_1}{\snorm{Z_1}}^\top Z_2 \sim N(0,\tau^2)$ and is independent of $Z_1$. Hence, using the standard Gaussian- tail bounds for $\delta>0, C\geq1$. 
    \begin{equation}
    \label{eqn: bound rotational invariant part}
        P\left(\frac{Z_1}{\snorm{Z_1}}^\top Z_2> \frac{\delta}{C} \right)\leq e^{-\frac{\delta^2}{2C^2\tau^2}}
    \end{equation}
    We can also get a concentration inequality for $\snorm{Z_1}$, using \cite{ledoux2006isoperimetry} formula (3.5).
    \begin{align}
        P\left(\snorm{Z_1}> C \right)&\leq 4e^{-\frac{C^2}{8\E(\snorm{Z_1}^2)}} \nonumber \\
        &\overset{(i)}=4e^{-\frac{C^2}{8g\tau^2}} \label{eqn: bound for norm part}
    \end{align}
    Where $(i)$ follows from the fact that $\snorm{Z_1}^2\tau^{-2}\sim\chi^2(g)$, hence $\E(\snorm{Z_1}^2)=g\tau^2$.\\
    Combining, \eqref{eqn: bound for norm part} and \eqref{eqn: bound rotational invariant part}, we have:
    $$P(Z_1^TZ_2>\delta)\leq 4\exp\left(-\frac{\delta^2}{2C^2\tau^2}-\frac{C^2}{8g\tau^2}\right)$$
    Setting $C^2=\sqrt{4g}\delta$ we get :
    $$P(Z_1^TZ_2>\delta)\leq 4\exp\left(-\frac{\delta}{\sqrt{4g}\tau^2}\right)$$
    One can convince themselves, that $Z_1^TZ_2$ has a symmetric distribution it is a product of Gaussians. Hence, we also have:
    $$P(Z_1^TZ_2<-\delta)\leq 4\exp\left(-\frac{\delta}{\sqrt{4g}\tau^2}\right)$$
    Using the union bound, 
    $$P(|Z_1^TZ_2|<\delta)\geq 1-8\exp\left(-\frac{\delta}{\sqrt{4g}\tau^2}\right)$$
    The proof for $\inner{\bzeta^{(i)},\bzeta^{(j)}}_{\bP_o}$ is identical.
\end{proof}
\begin{lemma}
    \label{lem: w.h.p shifts nearly orthogonal in off manifold space}
    For all $r\in [k]$ and $j\neq r$ consider the following events hold with high probability:
\begin{enumerate}
    \item $\frac{g\tau^2}{2}<\bzeta^{(r)\top}\bP\bzeta^{(r)}=\snorm{\bP\bzeta^{(r)}}^2<\frac{13g\tau^2}{8}$
    \item $\frac{D\tau^2}{2}<\snorm{\bzeta^{(r)}}^2<\frac{13D\tau^2}{8}$
    \item $\norm{\bP\bzeta^{(r)}}_{\infty}<  3\tau\sqrt{2\log(2g)}$
    \item $\max_{r\neq j}|\inner{\bzeta^{(r)},\bzeta^{(j)}}_{\bP}|\leq\frac{g\tau^2}{20k}-\frac{10k+1}{10k}\sqrt{\frac{13g\tau^2}{8}}=c_3(g,\tau,k)\asymp\left(\frac{g\tau^2}{k}\right)$
\end{enumerate}
    Specifically,
    \begin{align*}
        P\left(k\left(\max_{i\neq j}|\inner{\bzeta^{(i)},\bzeta^{(j)}}_{\bP}|+\epsilon^{\perp}_{\zeta} \right)<\frac{1}{10}\cdot[ \bzeta^{(r)\top}\bP\bzeta^{(r)}-\epsilon^{\perp}_{\zeta}]\right)\geq 1-\delta_1
    \end{align*}
    $$\delta_1=\left(4k(k-1)\exp{\left(-\frac{c_3}{\sqrt{4g}\tau^2}\right)}+2k\exp(-\frac{g}{16})+2k\exp(-\frac{D}{16})+k\exp(-4\log2g)\right) $$
    With all events (1-4) being satisfied.

\end{lemma}
\begin{proof} 
    If the above events hold true then $\epsilon^{\perp}_{\zeta}=\Bar{\omega}\norm{\bP\bzeta}\leq \sqrt{\frac{13g\tau^2}{8}}$.\\
    Using \lemref{lem: w.h.p shifts nearly orthogonal in off manifold space} with $\delta=c_3(g,\tau,k)$ and due to the assumption of $c_3(g,\tau,k)>0$ we have:
    \begin{align}
    \label{eqn: inner product small prob}
        P \left[\inner{\bzeta^{(r)},\bzeta^{(j)}}_{\bP} \geq c_3(g,\tau,k) \right] \leq 4\exp{-\frac{c_3}{\sqrt{4g}\tau^2}}\\
        P \left[\inner{\bzeta^{(r)},\bzeta^{(j)}}_{\bP} \leq -c_3(g,\tau,k) \right] \leq 4\exp{-\frac{c_3}{\sqrt{4g}\tau^2}} \nonumber
    \end{align}
Let,
\begin{align*}
    P(A)&=P \{\bigcup_{r\neq j}\left[\inner{\bzeta^{(r)},\bzeta^{(j)}}_{\bP} \geq c_3\cup \inner{\bzeta^{(r)},\bzeta^{(j)}}_{\bP}\leq -c_3  \right]\}\overset{(i)}\leq 4k(k-1)\exp{-\frac{c_3}{\sqrt{4g}\tau^2}}\\
    P(B)&=P \{\bigcup_{r\in[k]}\left[\bzeta^{(r)\top}\bP\bzeta^{(r)} \geq \frac{13g\tau^2}{8}\cup \bzeta^{(r)\top}\bP\bzeta^{(r)}\leq \frac{g\tau^2}{2}  \right]\}\overset{(ii)}\leq 2k\exp(-\frac{g}{16})\\
    P(C)&=P \{\bigcup_{r\in[k]}\left[\snorm{\bzeta^{(r)}}^2 \geq \frac{13D\tau^2}{8}\cup \snorm{\bzeta^{(r)}}^2\leq \frac{D\tau^2}{2}  \right]\}\overset{(iii)}\leq 2k\exp(-\frac{D}{16})\\
    P(D)&=P \{\bigcup_{r\in[k]}\norm{\bP\bzeta^{(r)}}_{\infty}>  3\tau\sqrt{2\log(2g)}\}\overset{(iv)}\leq k\exp(-4\log2g)
\end{align*}
Where $(i)$ uses \eqref{eqn: inner product small prob}; $(ii), (iii)$ uses \lemref{lem:norm/quadratic form bounded whp} and $(iv)$ uses \lemref{lem:linf of perturbation vectors bounded whp},  along with the Union bound.
The probability that all of the high-probability events happen can be computed by the union bound:
$1-P(A\cup B\cup C\cup D)\geq 1-\left(4k(k-1)\exp{(-\frac{c_3}{\sqrt{4g}\tau^2})}+2k\exp(-\frac{g}{16})+2k\exp(-\frac{D}{16})+k\exp(-4\log2g)\right)$.\\
Hence, 
\begin{align*}
    P\left(k\left(\max_{i\neq j}|\inner{\bzeta^{(i)},\bzeta^{(j)}}_{\bP}|+\epsilon^{\perp}_{\zeta} \right)<\frac{1}{10}\cdot[ \bzeta^{(r)\top}\bP\bzeta^{(r)}-\epsilon^{\perp}_{\zeta}]\right)\geq P\left(k\left(\max_{i\neq j}|\inner{\bzeta^{(i)},\bzeta^{(j)}}_{\bP}| \right)<c_3(g,\tau,k)\right)\\
    \geq 1-\left(4k(k-1)\exp{(-\frac{c_3}{\sqrt{4g}\tau^2})}+2k\exp(-\frac{g}{16})+2k\exp(-\frac{D}{16})+k\exp(-4\log2g)\right)
\end{align*}
\end{proof}

\begin{lemma}
\label{lem:shift/mean interaction is small}
    For some $q\in [k]$, let $V_q=\sum_{r\in[k]}y^{(r)}y^{(q)}\inner{\bmut^{(q)},\bzeta^{(r)}}_{\bP_o}$. Then w.h.p. $|V_q|<\frac{1}{10}\cdot c_4(d,\tau,k,p)$,
    where $c_4(d,\tau,k,p)=d-\Delta'-k(p+\Delta')+\frac{9}{10}\cdot\left(\frac{d\tau^2}{2}-\sqrt{\frac{13d\tau^2}{8}}(\sigma\sqrt{2}\ln d+1)\right)$
\end{lemma}
\begin{proof}
    As each $\bzeta^{(r)}\overset{i.i.d.}\sim N(0,\tau^2\bI_D)$ and $\snorm{\bmut}^2=d$, we have $y^{(r)}y^{(q)}\inner{\bmut^{(q)},\bzeta^{(r)}}_{\bP_o}\overset{i.i.d}\sim N(0,\tau^2d)$. Consequently, $V_q\sim N(0,\tau^2dk)$.\\
    Using gaussian tail bound we have:
    \begin{align*}
        P\left(|V_q|\geq \frac{1}{10}\cdot c_4(d,\tau,k,p)\right)\leq 2\exp\left(-\frac{c_4^2}{20\tau^2dk}\right)
    \end{align*}
\end{proof}

\begin{lemma}
    \label{lem: w.h.p shifts nearly orthogonal in  manifold space}
     For all $r\in [k]$ and $j\neq r$ consider the following events hold with high probability:
\begin{enumerate}
    \item $\frac{d\tau^2}{2}<\bzeta^{(r)\top}\bP_o\bzeta^{(r)}=\snorm{\bP_o\bzeta^{(r)}}^2<\frac{13d\tau^2}{8}$
    \item $\frac{D\tau^2}{2}<\snorm{\bzeta^{(r)}}^2<\frac{13D\tau^2}{8}$
    \item $\norm{\bP_o\bzeta^{(r)}}_{\infty}<  3\tau\sqrt{2\log(2d)}$
    \item $|V_r|<\frac{1}{10}\cdot c_4(d,\tau,k,p)$
    \item $\max_{r\neq j}|\inner{\bzeta^{(r)},\bzeta^{(j)}}_{\bP_o}|\leq\frac{d\tau^2}{20k}-\frac{10k+1}{10k}\sqrt{\frac{13d\tau^2}{8}}(\sigma\sqrt{2\ln d +1})=c_5(d,\tau,k)\asymp\left(\frac{d\tau^2}{k}\right)$
\end{enumerate}
    Specifically,
    \begin{align*}
        P\left(k\left(\max_{i\neq j}|\inner{\bzeta^{(i)},\bzeta^{(j)}}_{\bP_o}|+\epsilon^{\shortparallel}_{\zeta} \right)<\frac{1}{10}\cdot[ \bzeta^{(r)\top}\bP_o\bzeta^{(r)}-\epsilon^{\shortparallel}_{\zeta}]\right)\geq 1-\delta_2
    \end{align*}
    $$\delta_2=\left(4k(k-1)\exp{\left(-\frac{c_5}{\sqrt{4d}\tau^2}\right)}+2k\exp(-\frac{d}{16})+2k\exp(-\frac{D}{16})+k\exp(-4\log2d)+2k\exp(-\frac{c_4^2}{20\tau^2dk})\right) $$
    With all events (1-5) being satisfied.
\end{lemma}
\begin{proof}
    If the above events hold true then, $\epsilon^{\shortparallel}_{\zeta}=\left(\sqrt{2}\sigma\ln d+\Bar{\omega}\right)\norm{\bP_o\bzeta}\leq \sqrt{\frac{13d\tau^2}{8}}(\sigma\sqrt{2\ln d +1})$. We can mimic the proof of \lemref{lem: w.h.p shifts nearly orthogonal in off manifold space} replacing $\bP\to\bP_o,g\to d$ and using the corresponding probability bounds. Use \lemref{lem: semi-inner product close to zero} with $\delta=c_5(d,\tau,k)$ to get bound for event $A$. In addition to the events $A, B, C, D$ defined in \lemref{lem: w.h.p shifts nearly orthogonal in off manifold space} we need to subtract the probability of an event $E$ from our union bound so that the event 4 of this lemma is satisfied w.h.p.
    $$P(E)=P \{\bigcup_{r\in[k]}|V_r|> \frac{1}{10}\cdot c_4(d,\tau,k,p)\}\overset{(i)}\leq 2k\exp\left(-\frac{c_4^2}{20\tau^2dk}\right)$$
    Where $(i)$ is a consequence of \lemref{lem:shift/mean interaction is small}. Using the union bound $1-P(A)-P(B)\ldots -P(E)$, we have:
    \begin{align*}
    P\left(k\left(\max_{i\neq j}|\inner{\bzeta^{(i)},\bzeta^{(j)}}_{\bP_o}|+\epsilon^{\shortparallel}_{\zeta} \right)<\frac{1}{10}\cdot[ \bzeta^{(r)\top}\bP_o\bzeta^{(r)}-\epsilon^{\shortparallel}_{\zeta}]\right)\geq P\left(k\left(\max_{i\neq j}|\inner{\bzeta^{(i)},\bzeta^{(j)}}_{\bP_o}| \right)<c_5(g,\tau,k)\right)\\
    \geq 1-\left(4k(k-1)\exp{(-\frac{c_5}{\sqrt{4d}\tau^2})}+2k\exp(-\frac{d}{16})+2k\exp(-\frac{D}{16})+k\exp(-4\log2d)+2k\exp\left(-\frac{c_4^2}{20\tau^2dk}\right)\right)
\end{align*}
\end{proof}

\subsection{Generalisation}
Now onwards, define the index set of positive/negative second layer weights as $J+=\{j:v_j>0\},J-=\{j:v_j<0\}$.
\begin{lemma} \label{lem:bound Delta}
    We have $\Delta'\leq \frac{d}{21}$.
\end{lemma}
\begin{proof}
    Recall that $k \left( p + \Delta' + 1 \right) \leq \frac{d - \Delta' + 1 }{10}$. Since $k \geq 2$ and $p \geq 0$ it implies that $2(\Delta'+1) \leq \frac{d - \Delta'+ 1 }{10}$. Hence, $\Delta' \leq \frac{d-19}{21} \leq \frac{d}{21}$.
\end{proof}
 The proofs of the following four lemmas \ref{lem:lam upper bound} to \ref{lem:neuron input not small} can be directly translated from some of the lemmas proved in \cite{frei2023double} to our setting  The reader needs to just replace the $\bx,\bxi,\Delta$ terms in their proofs to $\bxt,\bxit,\Delta'$ in our setting. Under the hood, one can mimic the proofs by just replacing such terms because of the properties established in \lemref{lem:properties}. For the properties to hold, we need $\norm{\bzeta^{(r)}+\bomega}\leq \Bar{\zeta}$ for any $r\in [k]$. Note that this gets satisfied if $\snorm{\bzeta^{(r)}}<\frac{13D\tau^2}{8}$ which is true with probability $1-\exp(-D/16)$ (\lemref{lem:norm/quadratic form bounded whp}). So, our previously mentioned lemmas will hold w.h.p. for brevity's sake, we omit their proofs. However, if one wants, they can refer to the analogous lemmas in \cite{frei2023double} cited accordingly.
\begin{lemma}[Translated from Lemma A.10 \cite{frei2023double}] \label{lem:lam upper bound}
	If $\cs$ satisfies assumptions \ref{ass:dist}, then for all $q \in Q$ we have the following with probability $1-\exp(-D/16)$.
	\[
		\max\left\{ \sum_{i \in I^{(q)}} \sum_{j \in J_+} v_j^2 \lambda_i \phi'_{i,j}, \sum_{i \in I^{(q)}} \sum_{j \in J_-} v_j^2 \lambda_i \phi'_{i,j} \right\}
		\leq \frac{1}{(1-2c')(d-\Delta' + 1)}~.
	\]
\end{lemma}
\begin{lemma}[Translated from Lemma A.11 \cite{frei2023double}] \label{lem:lam lower bound}
	If $\cs$ satisfies assumptions \ref{ass:dist}, then for all $q \in Q_+$ we have the following with probability $1-\exp(-D/16)$.
	\[
		\sum_{i \in I^{(q)}} \sum_{j \in J_+} v_j^2 \lambda_i \phi'_{i,j} \geq \left( 1- \frac{c'}{1-2c'} \right) \frac{1}{3d+\Delta'+ 1}~,
	\]
	and for all $q \in Q_-$ we have 
	\[
		\sum_{i \in I^{(q)}} \sum_{j \in J_-} v_j^2 \lambda_i \phi'_{i,j} \geq \left( 1- \frac{c'}{1-2c'} \right) \frac{1}{3d+\Delta'+ 1}~.
	\]
\end{lemma}
\begin{proof}[Proof of Theorem \ref{thm:good example labeld correctly}]
\label{proof: main1}
	Working is shown for the case when $r \in Q_+$. The proof for $r \in Q_-$ is similar.
	By \eqref{eqn: main kkt}, for every $j \in [w]$ we have: 
	\begin{align} \label{eq:single neuron separating r 2}
		\bw_j^\top \bxt + b_j
		&= \left( \sum_{i \in [n]} \lambda_i y_i v_j \phi'_{i,j} \bxt_i^\top \bxt \right) + \sum_{i \in [n]} \lambda_i y_i v_j \phi'_{i,j} \nonumber
		\\
		&= \sum_{i \in [n]} \lambda_i y_i v_j \phi'_{i,j} (\bxt_i^\top \bxt + 1) \nonumber
		\\
		&= \left( \sum_{i \in I^{(r)}} \lambda_i v_j \phi'_{i,j} (\bxt_i^\top \bxt + 1)  \right) + \sum_{q \in [k] \setminus \{r\}} \sum_{i \in I^{(q)}} \lambda_i y_i v_j \phi'_{i,j} (\bxt_i^\top \bxt + 1)~.
	\end{align}

	Now,
	\begin{align} \label{eq:general neurons sum}
		\cn_\btheta(\bxt)
		&= \sum_{j \in [w]} v_j \phi(\bw_j^\top \bxt + b_j)
		\geq \sum_{j \in J_+} v_j (\bw_j^\top \bxt + b_j) + \sum_{j \in J_-} v_j \phi(\bw_j^\top \bxt + b_j)~.
	\end{align}
	By \eqref{eq:single neuron separating r 2} we have
	\begin{align*}
		\sum_{j \in J_+} v_j (\bw_j^\top \bxt + b_j) 
		&= \sum_{j \in J_+} \left[ \left( \sum_{i \in I^{(r)}} \lambda_i v_j^2 \phi'_{i,j} (\bxt_i^\top \bxt + 1) \right) + \sum_{q \in Q \setminus \{r\}} \sum_{i \in I^{(q)}} \lambda_i y_i v_j^2 \phi'_{i,j} (\bxt_i^\top \bxt + 1) \right]
		\\
		&\geq \sum_{j \in J_+} \left[ \left( \sum_{i \in I^{(r)}} \lambda_i v_j^2 \phi'_{i,j} (d - \Delta' + 1) \right) - \sum_{q \in [k] \setminus \{r\}} \sum_{i \in I^{(q)}} \lambda_i v_j^2 \phi'_{i,j} (p + \Delta'+ 1) \right]
		\\
		&= \left( (d - \Delta' + 1) \sum_{i \in I^{(r)}} \sum_{j \in J_+} \lambda_i v_j^2 \phi'_{i,j}  \right) 
		\\
		&\;\;\;\; - \left(  (p + \Delta' + 1) \sum_{q \in [k] \setminus \{r\}} \sum_{i \in I^{(q)}} \sum_{j \in J_+} \lambda_i v_j^2 \phi'_{i,j} \right)~.
	\end{align*}
	By \lemref{lem:lam upper bound} and \lemref{lem:lam lower bound} the above is at least
	\begin{align} \label{eq:general neurons sum part 1}
		(d - \Delta' + 1) &\left( 1- \frac{c'}{1-2c'} \right) \frac{1}{3d+\Delta' + 1} -  (p + \Delta' + 1) k \cdot \frac{1}{(1-2c')(d-\Delta' + 1)} \nonumber
		\\
		&=\left( 1- \frac{c'}{1-2c'} \right) \frac{d - \Delta'+ 1}{3d+\Delta'+ 1} -  (p + \Delta' + 1) c' \cdot \frac{d - \Delta' + 1}{p+\Delta'+1} \cdot  \frac{1}{(1-2c')(d-\Delta'+ 1)} \nonumber
		\\
		&= \left( 1- \frac{c'}{1-2c'} \right) \frac{d - \Delta' + 1}{3d+\Delta' + 1} - \frac{c'}{1-2c'}~.
	\end{align}
	
	Likewise, we have 
	\begin{align*}
		 \sum_{j \in J_-} v_j \phi(\bw_j^\top \bxt + b_j)
		 &=  \sum_{j \in J_-} v_j \phi \left(  \left( \sum_{i \in I^{(r)}} \lambda_i v_j \phi'_{i,j} (\bxt_i^\top \bxt + 1)  \right) + \sum_{q \in [k] \setminus \{r\}} \sum_{i \in I^{(q)}} \lambda_i y_i v_j \phi'_{i,j} (\bxt_i^\top \bxt + 1) \right)
		 \\
		 &\geq  \sum_{j \in J_-} v_j \phi \left(  \left( \sum_{i \in I^{(r)}} \lambda_i v_j \phi'_{i,j} (d - \Delta' + 1)  \right) + \sum_{q \in [k] \setminus \{r\}} \sum_{i \in I^{(q)}} \lambda_i | v_j | \phi'_{i,j} (p + \Delta'+ 1) \right)
		 \\
		 &\geq \sum_{j \in J_-} v_j \phi \left( \sum_{q \in [k] \setminus \{r\}} \sum_{i \in I^{(q)}} \lambda_i | v_j | \phi'_{i,j} (p + \Delta' + 1) \right)
		 \\
		 &= - \sum_{j \in J_-}  \sum_{q \in [k] \setminus \{r\}} \sum_{i \in I^{(q)}} \lambda_i v_j^2 \phi'_{i,j} (p + \Delta'+1)
		 \\
		 &= -  (p + \Delta'+ 1)  \sum_{q \in [k] \setminus \{r\}}  \sum_{i \in I^{(q)}} \sum_{j \in J_-} \lambda_i v_j^2 \phi'_{i,j}~.
	\end{align*}
	By  \lemref{lem:lam upper bound} the above is at least
	\begin{align} \label{eq:general neurons sum part 2}
		- (p + \Delta'+ 1) k \cdot  \frac{1}{(1-2c')(d-\Delta'+ 1)}
		&=- (p + \Delta'+ 1) c' \cdot \frac{d - \Delta'+ 1}{p+\Delta'+1} \cdot  \frac{1}{(1-2c')(d-\Delta'+ 1)} \nonumber
		\\
		&= - \frac{c'}{1-2c'}~. 
	\end{align}
	Combining \eqref{eq:general neurons sum},~(\ref{eq:general neurons sum part 1}), and~(\ref{eq:general neurons sum part 2}), we get
	\begin{align*}
		\cn_\btheta(\bxt)
		&\geq  \left( 1- \frac{c'}{1-2c'} \right) \frac{d - \Delta' + 1}{3d+\Delta'+ 1} - \frac{c'}{1-2c'} - \frac{c'}{1-2c'}~. 
	\end{align*}	
	Using $c' \leq \frac{1}{10}$ and $\Delta' \leq d/21$ (which holds by \lemref{lem:bound Delta}),  the above is at least 
	\[
		\frac{7}{8} \cdot \frac{d - \Delta' + 1}{3d+\Delta' + 1} - \frac{2}{8} 
		\geq \frac{7}{8} \cdot \frac{d - \Delta'}{3d+\Delta'} - \frac{2}{8}~.
	\]
	By \lemref{lem:bound Delta},
	the displayed equation is at least
	\[
		\frac{7}{8} \cdot \frac{d - d/21}{3d+d/21} - \frac{2}{8}
		= \frac{7}{8} \cdot \frac{5}{16} - \frac{2}{8}
		> 0~.
	\]	
\end{proof}
\subsection{Existence of Adversarial perturbations}
\begin{lemma}[Translated from Lemma D.2 \cite{frei2023double}] \label{lem:output not too large}
	Suppose $\cs$ satisfies assumption \ref{ass:dist}.  
 Let $\bxt \in \reals^D$ and $r \in Q$ such that $\bxt=\bmut^{(r)}+\bxit$ is a 'nice example'. Then, $| \cn_\btheta(\bxt) | \leq 2$ with probability $1-\exp(-D/16)$.
\end{lemma}
\begin{lemma}[Translated from Lemma D.3 \cite{frei2023double}] \label{lem:neuron input not small}
	Suppose $\cs$ satisfies assumption \ref{ass:dist}.  
 Let $\bxt \in \reals^D$ and $r \in Q$ such that $\bxt=\bmut^{(r)}+\bxit$ is a 'nice example'. 
	Then, for all $j \in [w]$ we have the following with probability $1-\exp(-D/16)$:
	\[
		\bw_j^\top \bxt + b_j \geq - \sum_{i \in [n]} \lambda_i | v_j | \phi'_{i,j} (2 \Delta'+ p + 1)~.
	\]
\end{lemma}

\begin{lemma} \label{lem:neuron input changes fast}
	Suppose $\cs$ satisfies assumptions \ref{ass:dist}. \\ 
 Let $\bu_{\perp} = \sum_{q\in [k]}y^{(q)}\bP\bzeta^{(q)}$, \, $\bu_{\shortparallel} = \sum_{q\in [k]}y^{(q)}(\bmut^{(q)}+\bP_{o}\bzeta^{(q)})$. Then $\bu_{\perp},\bu_{\shortparallel}$ satisfies the following with probability at least $1-\delta_1,1-\delta_2$ respectively.\\
 For every $j \in J_+$ we have :
 \begin{align*}
     \bw_j^\top \bu_{\perp} \geq \sum_{i \in [n]} \lambda_i v_j \phi'_{i,j}  c_6(g,\tau)>0; \,
     \bw_j^\top \bu_{\shortparallel} \geq \frac{9}{10}\sum_{i \in [n]} \lambda_i v_j \phi'_{i,j} c_4(d,\tau,k,p)>0 , 
 \end{align*}

	For every $j \in J_-$ we have :
	\begin{align*}
	     \bw_j^\top \bu_{\perp} \leq \sum_{i \in [n]} \lambda_i v_j \phi'_{i,j}  c_6(g,\tau) <0;\,
     \bw_j^\top \bu_{\shortparallel} \leq \frac{9}{10}\sum_{i \in [n]} \lambda_i v_j \phi'_{i,j}  c_4(d,\tau,k,p)<0, 
	\end{align*}
Where, $c_4(d,\tau,k,p)=d - \Delta'+\frac{9}{10}\cdot\left(\frac{d\tau^2}{2}-\sqrt{\frac{13d\tau^2}{8}}(\sigma\sqrt{2}\ln d+1)\right)- k(p + \Delta') $ and $c_6(g,\tau)=\frac{9}{10}\cdot(\frac{g\tau^2}{2}-\sqrt{\frac{13g\tau^2}{8}})$~.
$\delta_1,\delta_2$ are defined in \lemref{lem: w.h.p shifts nearly orthogonal in off manifold space} and \lemref{lem: w.h.p shifts nearly orthogonal in  manifold space} respectively.
\end{lemma}

\begin{proof}
Note that from \lemref{lem:rate for c4 and c6} the constants $c_4(d,\tau,k,p), c_6(g,\tau)$ are positive. 
\subsubsection*{On manifold perturbation}
For $\bw_j^\top \bu_{\shortparallel}$ we have:
	\begin{align} \label{eq:uz bound}
		\bw_j^\top &\sum_{r \in [k]} y^{(r)} (\bmut^{(r)}+\bP_{o}\bzeta^{(r)})
		= \sum_{i \in [n]} \lambda_i y_i v_j \phi'_{i,j} \bxt_i^\top  \sum_{r \in [k]} y^{(r)} (\bmut^{(r)}+\bP_o\bzeta^{(r)}) \nonumber
		\\
		&= \sum_{q \in [k]} \sum_{i \in I^{(q)}} \lambda_i y^{(q)} v_j \phi'_{i,j}  \bxt_i^\top  \left(  y^{(q)} (\bmut^{(q)}+\bP_o\bzeta^{(q)}) + \sum_{r \in [k] \setminus \{q\}} y^{(r)} (\bmut^{(r)}+\bP_o\bzeta^{(r)}) \right) \nonumber
		\\
		&= \sum_{q \in [k]} \sum_{i \in I^{(q)}} \lambda_i v_j \phi'_{i,j}  \left(  (y^{(q)})^2 \bxt_i^\top (\bmut^{(q)}+\bP_o\bzeta^{(q)}) + \sum_{r \in [k] \setminus \{q\}} y^{(q)} y^{(r)} \bxt_i^\top (\bmut^{(r)}+\bP_o\bzeta^{(r)}) \right) \nonumber
		\\
		&= \sum_{q \in [k]} \sum_{i \in I^{(q)}} \lambda_i v_j \phi'_{i,j}  [  (\bmut^{(q)})^\top \bmut^{(q)} + \bxit_i^\top \bmu^{(q)}+(\bmut^{(q)})^\top\bP_o\bzeta^{(q)} + (\bM\bxi_i+\bomega_i)^\top\bP_o\bzeta^{(q)}+\snorm{\bP_o\bzeta^{(q)}}^2 \nonumber
  \\ &\sum_{r \in [k] \setminus \{q\}} y^{(q)} y^{(r)}\left(  (\bmut^{(q)})^\top \bmut^{(r)} + \bxit_i^\top \bmut^{(r)} + \inner{\bmut^{(q)},\bzeta^{(r)}}_{\bP_o}+\inner{\bzeta^{(q)},\bzeta^{(r)}}_{\bP_o} +(\bM\bxi_i+\bomega_i)^\top\bP_o\bzeta^{(r)} \right) ]~. 
	\end{align}
	
	For $j \in J_+$ the above is at least
	\begin{align*}
	    &\sum_{q \in [k]} \sum_{i \in I^{(q)}} \lambda_i v_j \phi'_{i,j}  \left( V_q+d - \Delta'-\epsilon^{\shortparallel}_{\zeta}+\snorm{\bP_o\bzeta^{(q)}}^2 - \sum_{r \in [k] \setminus \{q\}} (p + \Delta'+\inner{\bzeta^{(q)},\bzeta^{(r)}}_{\bP_o}+\epsilon^{\shortparallel}_{\zeta}) \right)\\
     &\geq \sum_{q \in [k]} \sum_{i \in I^{(q)}} \lambda_i v_j \phi'_{i,j}  \left( V_q+d - \Delta'-\epsilon^{\shortparallel}_{\zeta}+\snorm{\bP_o\bzeta^{(q)}}^2 -  k(p + \Delta'+\max_{i\neq j}\inner{\bzeta^{(i)},\bzeta^{(j)}}_{\bP_o}+\epsilon^{\shortparallel}_{\zeta}) \right)\\
     &\overset{(i)}\geq \sum_{q \in [k]} \sum_{i \in I^{(q)}} \lambda_i v_j \phi'_{i,j}  \left( V_q+d - \Delta'+\frac{9}{10}\cdot(\snorm{\bP_o\bzeta^{(q)}}^2 - \epsilon^{\shortparallel}_{\zeta})- k(p + \Delta') \right)\\&\overset{(ii)}\geq \sum_{q \in [k]} \sum_{i \in I^{(q)}} \lambda_i v_j \phi'_{i,j}  \left( V_q+d - \Delta'+\frac{9}{10}\cdot\left(\frac{d\tau^2}{2}-\sqrt{\frac{13d\tau^2}{8}}(\sigma\sqrt{2}\ln d+1)\right)- k(p + \Delta') \right)\\
     &\overset{(iii)}\geq \frac{9}{10}c_4(d,\tau,k,p)\sum_{q \in [k]} \sum_{i \in I^{(q)}} \lambda_i v_j \phi'_{i,j}>0
	\end{align*}
	We use the \lemref{lem: w.h.p shifts nearly orthogonal in  manifold space} for the following:\\ $(i)$ we used that $\max_{i\neq j}\inner{\bzeta^{(i)},\bzeta^{(j)}}_{\bP_o}<\frac{1}{10}\cdot \snorm{\bP_o\bzeta^{(q)}}$ w.h.p.\\
 $(ii)$ we used that $\snorm{\bP_o\bzeta^{(q)}}>\frac{d\tau^2}{2}$ and $\epsilon^{\shortparallel}_{\zeta}\leq \sqrt{\frac{13d\tau^2}{8}}(\sigma\sqrt{2}\ln d+1)$ w.h.p.\\
 $(iii)$ we used that $V_q>-\frac{1}{10}c_4(d,\tau,k,p)$ w.h.p.\\
	Similarly, for $j \in J_-$, \eqref{eq:uz bound} is at most
	\begin{align*}
	    \frac{9}{10}c_4(d,\tau,k,p)\sum_{q \in [k]} \sum_{i \in I^{(q)}} \lambda_i v_j \phi'_{i,j}<0
	\end{align*}
\subsubsection*{Off manifold perturbation}
For $\bw_j^\top \bu_{\perp}$ we have:
	\begin{align} \label{eq:uz bound2}
		\bw_j^\top &\sum_{r \in [k]} y^{(r)} \bP\bzeta^{(r)}
		= \sum_{i \in [n]} \lambda_i y_i v_j \phi'_{i,j} \bxt_i^\top  \sum_{r \in [k]} y^{(r)} \bP\bzeta^{(r)} \nonumber
		\\
		&= \sum_{q \in [k]} \sum_{i \in I^{(q)}} \lambda_i y^{(q)} v_j \phi'_{i,j}  \bxt_i^\top  \left(  y^{(q)} \bP\bzeta^{(q)} + \sum_{r \in [k] \setminus \{q\}} y^{(r)} \bP\bzeta^{(r)} \right) \nonumber
		\\
		&= \sum_{q \in [k]} \sum_{i \in I^{(q)}} \lambda_i v_j \phi'_{i,j}  \left(  (y^{(q)})^2 \bxt_i^\top \bP\bzeta^{(q)} + \sum_{r \in [k] \setminus \{q\}} y^{(q)} y^{(r)} \bxt_i^\top \bP\bzeta^{(r)} \right) \nonumber
		\\
		&= \sum_{q \in [k]} \sum_{i \in I^{(q)}} \lambda_i v_j \phi'_{i,j}  \left[ \bomega_i^\top\bP\bzeta^{(q)}+\snorm{\bP\bzeta^{(q)}}^2+ \sum_{r \in [k] \setminus \{q\}} y^{(q)} y^{(r)}\left(  \inner{\bzeta^{(q)},\bzeta^{(r)}}_{\bP} +\bomega_i^\top\bP\bzeta^{(r)} \right) \right]~. 
	\end{align}
	For $j \in J_+$ the above is at least
	\begin{align*}
	    &\sum_{q \in [k]} \sum_{i \in I^{(q)}} \lambda_i v_j \phi'_{i,j}  \left( \snorm{\bP\bzeta^{(q)}}^2 -\epsilon^{\perp}_{\zeta}- \sum_{r \in [k] \setminus \{q\}} (\inner{\bzeta^{(q)},\bzeta^{(r)}}_{\bP}+\epsilon^{\perp}_{\zeta}) \right)\\
     &\geq \sum_{q \in [k]} \sum_{i \in I^{(q)}} \lambda_i v_j \phi'_{i,j}  \left( \snorm{\bP\bzeta^{(q)}}^2 -\epsilon^{\shortparallel}_{\zeta}-  k(\max_{i\neq j}\inner{\bzeta^{(i)},\bzeta^{(j)}}_{\bP}+\epsilon^{\perp}_{\zeta}) \right)\\
     &\overset{(iv)}\geq \sum_{q \in [k]} \sum_{i \in I^{(q)}} \lambda_i v_j \phi'_{i,j}  \left( \frac{9}{10}\cdot(\snorm{\bP\bzeta^{(q)}}^2 - \epsilon^{\perp}_{\zeta})\right)\\
     &\overset{(v)}\geq \sum_{q \in [k]} \sum_{i \in I^{(q)}} \lambda_i v_j \phi'_{i,j}  \left(\frac{9}{10}\cdot\left(\frac{g\tau^2}{2}-\sqrt{\frac{13g\tau^2}{8}}\right) \right)=\sum_{q \in [k]} \sum_{i \in I^{(q)}} \lambda_i v_j \phi'_{i,j} c_6(g,\tau)>0
	\end{align*}
 We use the \lemref{lem: w.h.p shifts nearly orthogonal in off manifold space} for the following:\\ $(iv)$ we used that $\max_{i\neq j}\inner{\bzeta^{(i)},\bzeta^{(j)}}_{\bP}<\frac{1}{10}\cdot \snorm{\bP\bzeta^{(q)}}$ w.h.p.\\
 $(v)$ we used that $\snorm{\bP\bzeta^{(q)}}>\frac{g\tau^2}{2}$ and $\epsilon^{\perp}_{\zeta}\leq \sqrt{\frac{13g\tau^2}{8}}$ w.h.p.\\
 	Similarly, for $j \in J_-$, \eqref{eq:uz bound2} is at most
	\begin{align*}
	    c_6(g,\tau)\sum_{q \in [k]} \sum_{i \in I^{(q)}} \lambda_i v_j \phi'_{i,j}<0
	\end{align*}

\end{proof}
\begin{lemma}
\label{lem:rate for c4 and c6}
    $$\frac{18}{21}\cdot \frac{g\tau^2}{2}\leq c_6(g,\tau)\leq \frac{9}{10}\cdot \frac{g\tau^2}{2}$$
    and, 
    $$\frac{18}{21}\cdot d(1+\frac{\tau^2}{2})\leq c_4(d,\tau,k,p)\leq d(1+\frac{9\tau^2}{20})$$
\end{lemma}
\begin{proof}
    From assumption \ref{ass:dist}, $c_3(g,\tau,k)>0$:\\
        $\frac{g\tau^2}{20k}-\frac{10k+1}{10k}\sqrt{\frac{13g\tau^2}{8}}>0$ but $k\geq 2$ implying
        $\frac{g\tau^2}{42}\geq\sqrt{\frac{13g\tau^2}{8}}$.\\
        $$\frac{9}{10}\cdot\frac{g\tau^2}{2}\geq c_6(g,\tau)=\frac{9}{10}\cdot\left(\frac{g\tau^2}{2}-\sqrt{\frac{13g\tau^2}{8}}\right)\geq\frac{18}{21}\cdot \frac{g\tau^2}{2} $$
From assumption \ref{ass:dist}, $c_5(d,\tau,k)>0$:\\
$\frac{d\tau^2}{20k}-\frac{10k+1}{10k}\sqrt{\frac{13d\tau^2}{8}}(\sigma\sqrt{2\ln d +1})>0$, and $k\geq 2$ implying 
\begin{equation}
    \frac{d\tau^2}{42}\geq \sqrt{\frac{13d\tau^2}{8}}(\sigma\sqrt{2\ln d +1})
    \label{eq:bounded by d tau^2}
\end{equation}
    \begin{align*}
        c_4(d,\tau,k,p)&= d - \Delta'+\frac{9}{10}\cdot\left(\frac{d\tau^2}{2}-\sqrt{\frac{13d\tau^2}{8}}(\sigma\sqrt{2}\ln d+1)\right)- k(p + \Delta')\\
        &\overset{(i)}\geq d - \Delta'+\frac{18}{21}\cdot\frac{d\tau^2}{2}- k(p + \Delta'+1)\\
        &\overset{(ii)}\geq \frac{9}{10}\cdot(d-\Delta')+\frac{1}{10}+\frac{18}{21}\cdot\frac{d\tau^2}{2}\\
        &\overset{(iii)}\geq \frac{18}{21}\cdot d+\frac{18}{21}\cdot\frac{d\tau^2}{2}=\frac{18}{21}\cdot d(1+\frac{\tau^2}{2})
    \end{align*}
    For $(i)$ use \eqref{eq:bounded by d tau^2}; for $(ii)$ use \ref{ass:dist} $\frac{k(p+\Delta'+1)}{d-\Delta'+1}\leq \frac{1}{10}$; for $(iii)$ use \lemref{lem:bound Delta}.
    The upper-bound for $c_4(d,\tau,k,p)$ is straight forward and can be obtained by omitting the negative terms.
\end{proof}
\begin{lemma} \label{lem:neurons input become positive}
	Suppose $\cs$ satisfies assumption \ref{ass:dist}.  
 Let $\bxt \in \reals^D$ and $r \in [k]$ such that $\bxt=\bmut^{(r)}+\bxit$ is a 'nice example'. 
 Let $\bz_{\perp} = \eta^{\perp} \sum_{q\in [k]}y^{(q)}\bP\bzeta^{(q)}, \, \bz_{\shortparallel} =\eta^{\shortparallel} \sum_{q\in [k]}y^{(q)}\bmut^{(q)}$,  for $\eta^{\perp} \geq  \frac{ 2 \Delta' + p + 1}{c_6(g,\tau)}, \, \eta^{\shortparallel} \geq  \frac{10}{9}\cdot\frac{ 2 \Delta' + p + 1}{c_4(d,\tau,k,p)}$. Then, $\bz_{\perp},\bz_{\shortparallel}$ satisfies the following with probability at least $1-\delta_1,1-\delta_2$ respectively.\\
	 For all $j \in J_-$ we have:
 \begin{align*}
     \bw_j^\top (\bxt - \bz_{\perp}) + b_j\geq 0; \,
     \bw_j^\top (\bxt - \bz_{\shortparallel}) + b_j \geq 0
 \end{align*}
 and for all $j \in J_+$ : \begin{align*}
     \bw_j^\top (\bxt + \bz_{\perp}) + b_j\geq 0;\,
     \bw_j^\top (\bxt + \bz_{\shortparallel}) + b_j \geq 0,
 \end{align*}
$\delta_1,\delta_2$ are defined in \lemref{lem: w.h.p shifts nearly orthogonal in off manifold space} and \lemref{lem: w.h.p shifts nearly orthogonal in  manifold space} respectively.
\end{lemma}
\begin{proof}
	Let $j \in J_-$. By \lemref{lem:neuron input not small} w.h.p , we have 
	\[
		\bw_j^\top \bxt + b_j \geq 
		\sum_{i \in I} \lambda_i  v_j \phi'_{i,j} (2 \Delta'+ p + 1)~.
	\] 
	By \lemref{lem:neuron input changes fast}, we have
	\[
		-\bw_j^\top \bz_{\perp} \geq -\eta^{\perp}\sum_{i \in [n]} \lambda_i v_j \phi'_{i,j}  c_6(g,\tau) ;\quad
     -\bw_j^\top \bz_{\shortparallel} \geq -\eta^{\shortparallel}\frac{9}{10}\sum_{i \in [n]} \lambda_i v_j \phi'_{i,j}  c_4(d,\tau,k,p)
	\]
	Combining the last three displayed equations, we get
	\begin{align*}
		\bw_j^\top (\bxt - \bz_{\perp}) + b_j
		&= \bw_j^\top \bxt + b_j - \bw_j^\top \bz_{\perp}
		\\
		&\geq \sum_{i \in [n]} \lambda_i  v_j \phi'_{i,j} (2 \Delta' + p + 1) - \eta^{\perp}\sum_{i \in [n]} \lambda_i v_j \phi'_{i,j}  c_6(g,\tau)
		\\
		&= \sum_{i \in [n]} \lambda_i  v_j \phi'_{i,j} \left( 2 \Delta' +p + 1 - c_6(g,\tau)\eta^{\perp} \right)~. 
	\end{align*}
 Similarly for $\bz^{\shortparallel}$:
 \begin{align*}
		\bw_j^\top (\bxt - \bz_{\shortparallel}) + b_j
		\geq \sum_{i \in [n]} \lambda_i  v_j \phi'_{i,j} \left( 2 \Delta' +p + 1 - \frac{9}{10}c_4(d,\tau,k,p)\eta^{\shortparallel} \right)~. 
	\end{align*}
	Note that by \lemref{lem:rate for c4 and c6} we have $c_6(g,\tau),\, c_4(d,\tau,k,p) > 0$. Hence, for $\eta^{\perp} \geq  \frac{ 2 \Delta' + p + 1}{c_6(g,\tau)}, \, \eta^{\shortparallel} \geq  \frac{10}{9}\cdot\frac{ 2 \Delta' + p + 1}{c_4(d,\tau,k,p)}$ we have $\bw_j^\top (\bxt - \bz^{\perp}) + b_j, \, \bw_j^\top (\bxt - \bz^{\shortparallel}) + b_j\geq 0$ respectively.
	
	The proof for $j \in J_+$ is similar. Namely, by Lemmas~\ref{lem:neuron input not small} and~\ref{lem:neuron input changes fast}, we have
	\begin{align*}
		\bw_j^\top (\bxt + \bz_{\perp}) + b_j
		&= \bw_j^\top \bxt + b_j + \bw_j^\top \bz_{\perp}
		\\
		&-\geq \sum_{i \in [n]} \lambda_i  v_j \phi'_{i,j} (2 \Delta' + p + 1) +\eta^{\perp}\sum_{i \in [n]} \lambda_i v_j \phi'_{i,j}  c_6(g,\tau)
		\\
		&= \sum_{i \in [n]} \lambda_i  v_j \phi'_{i,j} \left( -2 \Delta' -p - 1 + c_6(g,\tau)\eta^{\perp} \right)~. 
	\end{align*}
	and hence for $\eta^{\perp} \geq  \frac{ 2 \Delta' + p + 1}{c_6(g,\tau)}$ we get $\bw_j^\top (\bxt + \bz^{\perp}) + b_j \geq 0$. Finally using identical arguments, for $\eta^{\shortparallel} \geq  \frac{10}{9}\cdot\frac{ 2 \Delta' + p + 1}{c_4(d,\tau,k,p)}$ we get $\bw_j^\top (\bxt + \bz^{\shortparallel}) + b_j \geq 0$
\end{proof}

\begin{proof}[Proof of Theorem \ref{thm:pert flips}]
\label{proof: main2}
We showcase the explicit steps for off-manifold  $\bz_{\perp}=(\eta_1^{\perp}+\eta_2^{\perp})\bu_{\perp}$, identical steps can be implemented to get the result for on-manifold $\bz_{\shortparallel}=(\eta_1^{\shortparallel}+\eta_2^{\shortparallel})\bu_{\shortparallel}$ and total $ \bz=(\eta_1+\eta_2)\bu$ attacks.\\
    We assume $r\in Q_+$, the case where $r\in Q_-$ can be proved similarly.
	We denote $\bxt' = \bxt - \bz_{\perp}$.
	By \lemref{lem:neuron input changes fast}, for every $j \in J_+$ we have 
	\begin{align*}
		\bw_j^\top  \bxt' + b_j 
		&= \bw_j^\top \bxt + b_j - \bw_j^\top   (\eta_1^{\perp} + \eta_2^{\perp})\bu_{\perp}
		\\
		&\leq \bw_j^\top \bxt + b_j  - (\eta_1^{\perp} + \eta_2^{\perp})   \sum_{i \in [n]} \lambda_i v_j \phi'_{i,j} c_6(g,\tau) ~.
	\end{align*}
	\begin{equation} \label{eq:value does not increase1}
	    \bw_j^\top  \bxt' + b_j 
	    \leq \bw_j^\top \bxt + b_j. 
	\end{equation}
	Thus, in the neurons $J_+$ the input does not increase when moving from $\bxt$ to $\bxt'$.
	
	Consider now $j \in J_-$. Let $\bxt_* = \bx - \eta_1^{\perp} \bu_{\perp}$. Setting $\eta_1^{\perp}=\frac{ 2 \Delta' + p + 1}{c_6(g,\tau)}$ and using \lemref{lem:neurons input become positive}, we have 
	$\bw_j^\top \bxt_* + b_j \geq 0$. Also, by \lemref{lem:neuron input changes fast}, we have 
	\begin{align*}
		\bw_j^\top  \bxt_* + b_j 
		&= \bw_j^\top \bxt + b_j - \bw_j^\top \eta_1^{\perp} \bu_{\perp}
		\\
		&\geq \bw_j^\top \bxt + b_j  - \eta_1^{\perp} \sum_{i \in [n]} \lambda_i v_j \phi'_{i,j}  c_6(g,\tau)~,
	\end{align*}
	the above is at least $\bw_j^\top \bxt + b_j$. Thus, when moving from $\bxt$ to $\bxt_*$ the input to the neurons $J_-$ can only increase, and at $\bxt_*$ it is 
	non-negative. 
	
	Next, we move from $\bxt_*$ to $\bxt'$. We have
	\begin{align*}
		\bw_j^\top \bxt' + b_j
		= \bw_j^\top \bxt_* + b_j - \eta_2^{\perp} \bw_j^\top   \bu_{\perp}
		\geq \max \left\{0,   \bw_j^\top \bxt + b_j  \right\} - \eta_2^{\perp} \bw_j^\top   \bu_{\perp}.
	\end{align*}
	By \lemref{lem:neuron input changes fast}, the above is at least 
	\begin{equation} \label{eq:value increases a lot1}
		 \max \left\{0, \bw_j^\top \bxt + b_j  \right\} - \eta_2^{\perp} \sum_{i \in [n]} \lambda_i v_j \phi'_{i,j}  c_6(g,\tau)
		 \geq 0~,
	\end{equation}

	Overall, we have 
	\begin{align*}
		\cn_\btheta(\bxt') 
		&= \left[ \sum_{j \in J_+} v_j \phi(\bw_j^\top \bxt' + b_j) \right] + \left[ \sum_{j \in J_-}  v_j \phi(\bw_j^\top \bxt' + b_j) \right]
		\\
		&\overset{(i)} = \left[ \sum_{j \in J_+} v_j \phi(\bw_j^\top \bxt' + b_j) \right] + \left[ \sum_{j \in J_-}  v_j (\bw_j^\top \bxt' + b_j) \right]
		\\
		&\overset{(ii)}  \leq \left[ \sum_{j \in J_+} v_j \phi(\bw_j^\top \bxt + b_j) \right] + \\
		&\;\;\;\;\; \left[ \sum_{j \in J_-}  v_j \left(   \max \left\{0,   \bw_j^\top \bxt + b_j  \right\} - \eta_2^{\perp} \sum_{i \in [n]} \lambda_i v_j \phi'_{i,j}  c_6(g,\tau) \right) \right]~,
	\end{align*}
    where in $(i)$ we used \eqref{eq:value increases a lot1}, and in $(ii)$ we used both \eqref{eq:value does not increase1} and \eqref{eq:value increases a lot1}.
    Now, the above equals	
	\begin{align*}
		&\left[ \sum_{j \in [w]} v_j \phi(\bw_j^\top \bxt + b_j) \right] - \left[ \sum_{j \in J_-}  v_j \eta_2^{\perp}  \sum_{i \in [n]} \lambda_i v_j \phi'_{i,j}  c_6(g,\tau)\right]
		\\
		&= \cn_\btheta(\bxt) - \eta_2^{\perp}  c_6(g,\tau)  \left[ \sum_{q' \in [k]} \sum_{i \in I^{(q')}} \sum_{j \in J_-} \lambda_i v_j^2 \phi'_{i,j}  \right]~.
	\end{align*}
	
	Combining the above with \lemref{lem:output not too large}, \lemref{lem:lam lower bound}, and \lemref{lem:neuron input changes fast}, we get 
	\begin{align*}
		\cn_\btheta(\bxt') 
		&\leq 2 - \eta_2^{\perp}    c_6(g,\tau) \left[ \sum_{q' \in Q_-} \sum_{i \in I^{(q')}} \sum_{j \in J_-} \lambda_i v_j^2 \phi'_{i,j}  \right]
		\\
		&\leq 2 - \eta_2^{\perp}    c_6(g,\tau) | Q_- |  \left( 1- \frac{c'}{1-2c'} \right) \frac{1}{3d+\Delta' + 1}
		\\
		&\leq 2 - \eta_2^{\perp}    c_6(g,\tau)  ck  \left( \frac{1-3c'}{1-2c'} \right) \frac{1}{3d+\Delta' + 1}~.
	\end{align*}
	For 
	\[
		\eta_2^{\perp} 
		= \frac{3(3d+\Delta'+1)(1-2c')}{  c_6(g,\tau) \cdot(1-3c')ck} 
	\] 
	we conclude that $\cn_\btheta(\bx')$ is at most $-1$. When $r\in Q_-$ one can get an identical result with a perturbation $+\bz_{\perp}$ instead of $-\bz_{\perp}$.\\
 
 Hence, an off-manifold perturbation $\bz_{\perp}=(\eta_1^{\perp}+\eta_2^{\perp})\bu_{\perp}$ where ,\begin{equation}
    \eta_1^{\perp}=\frac{ 2 \Delta' + p + 1}{c_6(g,\tau)}, \,\eta_2^{\perp}=\frac{3(3d+\Delta'+1)(1-2c')}{  c_6(g,\tau)\cdot (1-3c')ck} 
 \end{equation}  can flip the sign of the neuron accordingly.\\
 If we follow the identical steps as above for  on manifold perturbation $\bz_{\shortparallel}=(\eta_1^{\shortparallel}+\eta_2^{\shortparallel})\bu_{\shortparallel}$, we get that:
 \begin{align}
\eta_1^{\shortparallel}&=\frac{10}{9}\cdot\frac{ 2 \Delta' + p + 1}{c_4(d,\tau,k,p)}, \,\eta_2^{\shortparallel}=\frac{10}{9}\cdot\frac{3(3d+\Delta'+1)(1-2c')}{ c_4(d,\tau,k,p)\cdot (1-3c')ck}
 \end{align}
\subsubsection*{Off manifold perturbation upper bounds}
\begin{align}
\label{eqn: const_perp upper bound}
		(\eta^{\perp}_1+ \eta^{\perp}_2)
		&= \left( \frac{ 2 \Delta' + p + 1}{c_6(g,\tau)} +  \frac{3(3d+\Delta'+1)(1-2c')}{c_6(g,\tau)\cdot(1-3c')ck} \right) \nonumber
		\\
		&\leq \left( \frac{  2 (p + \Delta' + 1)}{c_6(g,\tau)} +  \frac{3(3d+\Delta'+1)(1-2c')}{c_6(g,\tau)\cdot(1-3c')ck} \right)\nonumber
		\\
		&\overset{(i)} = \left( \frac{2 c' (d-\Delta'+1)/k }{\frac{18}{42}\cdot g\tau^2} +  \frac{3(3d+\Delta'+1)(1-2c')}{\frac{18}{42}\cdot g\tau^2\cdot(1-3c')ck} \right)\nonumber
		\\
		&\overset{(ii)} \leq \left(  \frac{1}{k} \cdot \frac{ \frac{2}{10} (d+1) }{\frac{18}{42}\cdot g\tau^2} + \frac{1}{ck} \cdot \frac{3(3d + \frac{d}{21} + 1)}{\frac{18}{42}\cdot g\tau^2(1-\frac{3}{10})} \right)\nonumber
		\\
		&\leq \left( \frac{1}{k} \cdot \co{\left(\frac{d}{g\tau^2}\right) }+ \frac{1}{ck} \cdot  \co{\left(\frac{d}{g\tau^2}\right) } \right)\nonumber
		\\
		&\leq  \co{\left(\frac{d}{ckg\tau^2}\right) }~,
	\end{align}
 where in $(i)$ we used $k = c' \cdot \frac{d - \Delta' + 1}{p+\Delta'+1}$ and $\frac{18}{21}\cdot\frac{g\tau^2}{2}\leq c_6(g,\tau)$ from \lemref{lem:rate for c4 and c6}. In $(ii)$ we used \lemref{lem:bound Delta} and $c'<1/10$.
\begin{align*}
		\norm{\bu_{\perp}}^2=\norm{ \sum_{q \in [k]} y^{(q)} \bP\bzeta^{(q)}}^2
		&= \sum_{r \in [k]} \sum_{q \in [k]} y^{(r)} y^{(q)} \inner{\bzeta^{(r)}, \bzeta^{(q)}}_{\bP}
		\\
		&= \sum_{r \in [k]} \left[\inner{\bzeta^{(r)}, \bzeta^{(r)}}_{\bP} + \sum_{q \neq r} y^{(r)} y^{(q)} \inner{\bzeta^{(r)}, \bzeta^{(q)}}_{\bP} \right]
		\\
		&\overset{(iii)}\leq k\cdot\frac{13g\tau^2}{8} + k^2 c_3(g,\tau,k)
		\\
		&\leq k\cdot\frac{13g\tau^2}{8} + k\frac{g\tau^2}{20}= \co(kg\tau^2)~.
	\end{align*}
 \begin{equation}
 \label{eqn: u_perp upper bound}
     \norm{\bu_{\perp}}=\co{\left(\sqrt{kg\tau^2}\right)}
 \end{equation}
 where in $(iii)$ we used $\snorm{\bP\bzeta^{(r)}}<\frac{13g\tau^2}{8}; \,\inner{\bzeta^{(r)}, \bzeta^{(q)}}_{\bP}<c_3(g,\tau,k)$ \lemref{lem: w.h.p shifts nearly orthogonal in off manifold space} .\\
 Combining \eqref{eqn: const_perp upper bound} and \eqref{eqn: u_perp upper bound} we have $\norm{\bz_{\perp}}=\co{\left(\frac{d}{c\sqrt{kg\tau^2}}\right)}$.
 \begin{align}
 \label{eqn:u_perp inf upper bound}
		\norm{\bu_{\perp}}_{\infty}=\norm{ \sum_{q \in [k]} y^{(q)} \bP\bzeta^{(q)}}_{\infty}
		&\leq  \sum_{q \in [k]}\norm{\bP\bzeta^{(q)}}_{\infty}\nonumber
		\\
		&\overset{(iv)}\leq 3k\cdot\tau\sqrt{2\log(2g)}=\co{\left(k\tau\sqrt{2\log (2g)}\right)}~.
	\end{align}
 For $(iv)$ use, $\snorm{\bP\bzeta^{(q)}}_{\infty}\leq 3\tau\sqrt{2\log(2g)}$ (\lemref{lem: w.h.p shifts nearly orthogonal in off manifold space}).\\
 Combining \eqref{eqn: const_perp upper bound} and \eqref{eqn:u_perp inf upper bound} we have $\norm{\bz_{\perp}}_{\infty}=\co{\left(\frac{d\sqrt{2\log (2g)}}{cg\tau}\right)}$.
 \subsubsection*{On manifold perturbation upper bounds}
 \begin{align}
\label{eqn: const_parallel upper bound}
		(\eta^{\shortparallel}_1+ \eta^{\shortparallel}_2)
		&= \frac{10}{9}\left( \frac{ 2 \Delta' + p + 1}{c_4(d,\tau,k,p)} +  \frac{3(3d+\Delta'+1)(1-2c')}{c_4(d,\tau,k,p)\cdot(1-3c')ck} \right) \nonumber
		\\
		&\leq \frac{10}{9}\left( \frac{  2 (p + \Delta' + 1)}{c_4(d,\tau,k,p)} +  \frac{3(3d+\Delta'+1)(1-2c')}{c_4(d,\tau,k,p)\cdot(1-3c')ck} \right)\nonumber
		\\
		&\overset{(i)} = \left( \frac{2 c' (d-\Delta'+1)/k }{\frac{18}{21}\cdot d(1+\frac{\tau^2}{2})} +  \frac{3(3d+\Delta'+1)(1-2c')}{\frac{18}{21}\cdot d(1+\frac{\tau^2}{2})\cdot(1-3c')ck} \right)\nonumber
		\\
		&\overset{(ii)} \leq \left(  \frac{1}{k} \cdot \frac{ \frac{2}{10} (d+1) }{\frac{18}{21}\cdot d(1+\frac{\tau^2}{2})} + \frac{1}{ck} \cdot \frac{3(3d + \frac{d}{21} + 1)}{\frac{18}{21}\cdot d(1+\frac{\tau^2}{2})(1-\frac{3}{10})} \right)\nonumber
		\\
		&\leq \left( \frac{1}{k} \cdot \co{\left(\frac{1}{1+\frac{\tau^2}{2}}\right) }+ \frac{1}{ck} \cdot  \co{\left(\frac{1}{1+\frac{\tau^2}{2}}\right) } \right)\nonumber
		\\
		&\leq  \co{\left(\frac{1}{ck(2+\tau^2)}\right) }~,
	\end{align}
 where in $(i)$ we used $k = c' \cdot \frac{d - \Delta' + 1}{p+\Delta'+1}$ and $\frac{18}{21}\cdot d(1+\frac{\tau^2}{2})\leq c_4(d,\tau,k,p)$ from \lemref{lem:rate for c4 and c6}. In $(ii)$ we used \lemref{lem:bound Delta} and $c'<1/10$.
 \begin{align}
 \label{eqn: u_parallel upper bound}
		\norm{\bu_{\shortparallel}}^2&=\norm{ \sum_{q \in [k]} y^{(q)} (\bmut^{(q)}+\bP_o\bzeta^{(q)})}^2 \nonumber\\
		&= \sum_{r \in [k]} \sum_{q \in [k]} y^{(r)} y^{(q)} \inner{\bzeta^{(r)}, \bzeta^{(q)}}_{\bP_o}+ \sum_{r \in [k]} \sum_{q \in [k]} y^{(r)} y^{(q)} \inner{\bmut^{(r)}, \bmut^{(q)}}+\sum_{r \in [k]} \sum_{q \in [k]} 2y^{(r)} y^{(q)}   \inner{\bmut^{(r)}, \bzeta^{(q)}}_{\bP_o}
		\nonumber\\
		&= \sum_{r \in [k]} \left[\inner{\bzeta^{(r)}, \bzeta^{(r)}}_{\bP_o}+d + \sum_{q \neq r} y^{(r)} y^{(q)} \left(\inner{\bzeta^{(r)}, \bzeta^{(q)}}_{\bP_o} +\inner{\bmut^{(r)}, \bmut^{(q)}}\right)\right]+\sum_{r\in [k]} V_r
		\nonumber\\
		&\overset{(iii)}\leq k\cdot\frac{13d\tau^2}{8}+kd+k^2 c_5(d,\tau,k) + k^2 p
		+\frac{k}{10}\cdot c_4(d,\tau,k,p)\nonumber\\
		&\overset{(iv)}\leq k\cdot\frac{13d\tau^2}{8}+kd+k\frac{d\tau^2}{20} + k^2 p
		+\frac{k}{10}\cdot d(1+\frac{\tau^2}{2})\nonumber\\
  &\overset{(v)}\leq k\cdot\frac{13d\tau^2}{8}+kd+k\frac{d\tau^2}{20} + \frac{kd+k}{10}
		+\frac{k}{10}\cdot d(1+\frac{\tau^2}{2})=kd\left(\frac{6}{5}+\frac{69\tau^2}{40}\right)+\frac{k}{10}\nonumber\\
  &\leq \co{\left(kd(2+\tau^2)\right)}
	\end{align}
 For $(iii)$ note that $\snorm{\bP_o\zeta^{(r)}}<\frac{13d\tau^2}{8}$ ,$\inner{\bzeta^{(r)}, \bzeta^{(q)}}_{\bP_o}<c_5(d,\tau,k), |V_r|<\frac{1}{10}\cdot c_4(d,\tau,k,p)$ (\lemref{lem: w.h.p shifts nearly orthogonal in  manifold space}).\\
 For $(iv)$ , we know that $c_5(d,\tau,k)<\frac{d\tau^2}{20k}$ and $c_4(d,\tau,k,p)<d(1+\tau^2/2)$ (\lemref{lem:rate for c4 and c6}).\\
 For $(v)$, $k(p+\Delta'+1)<1/10\cdot (d-\Delta'+1)\implies kp<1/10\cdot (d+1)$.\\
  Combining \eqref{eqn: const_parallel upper bound} and \eqref{eqn: u_parallel upper bound} we have $\norm{\bz_{\shortparallel}}=\co{\left(\sqrt{\frac{d}{c^2k(2+\tau^2)}}\right)}$

  \begin{align}
 \label{eqn:u_parallel inf upper bound}
		\norm{\bu_{\shortparallel}}_{\infty}=\norm{ \sum_{q \in [k]} y^{(q)} (\bmut^{(q)}+\bP_o\bzeta^{(q)})}_{\infty}
		&\leq  \sum_{q \in [k]}\underbrace{\norm{\bmut^{(q)}}_{\infty}}_{\leq\snorm{\bmut^{(q)}}}+\sum_{q \in [k]}\norm{\bP_o\bzeta^{(q)}}_{\infty}\nonumber
		\\
		&\overset{(iv)}\leq k\sqrt{d}+3k\cdot\tau\sqrt{2\log(2d)}=\co{\left(kd+k\tau\sqrt{2\log (2d)}\right)}~.
	\end{align}
 For $(iv)$ use, $\snorm{\bP_o\bzeta^{(q)}}_{\infty}\leq 3\tau\sqrt{2\log(2d)}$ (\lemref{lem: w.h.p shifts nearly orthogonal in manifold space}).\\
 Combining \eqref{eqn: const_parallel upper bound} and \eqref{eqn:u_parallel inf upper bound} we have $\norm{\bz_{\shortparallel}}_{\infty}=\co{\left(\frac{d+\tau\sqrt{2\log (2d)}}{c(2+\tau^2)}\right)}$.
 The corresponding probability with all of the bounds holding true for the off-manifold and on-manifold perturbation is $1-\delta_1,1-\delta_2$, respectively. Where $\delta_1, \delta_2$ are defined in \lemref{lem: w.h.p shifts nearly orthogonal in off manifold space},\lemref{lem: w.h.p shifts nearly orthogonal in manifold space} respectively.
\end{proof}
\section{Additional Experiment Results}
\label{Additional details exp}

\begin{table*}[]

\centering
\begin{tabular}{lllllllllll}
\toprule
CIFAR Class & 0 & 1 & 2 & 3 & 4 & 5 & 6 & 7 & 8 & 9\\\hline
LPCA                                      & 8                        & 11                     & 8                      & 11                     & 9                      & 13                     & 7                      & 14                     & 10                     & 16                     \\ 
MLE(k=5)                                  & 18.34                    & 21.20                  & 21.35                  & 21.08                  & 21.43                  & 22.11                  & 22.96                  & 22.57                  & 19.80                  & 24.42                  \\ 
TwoNN                                     & 21.77                    & 25.98                  & 25.53                  & 26.09                  & 24.60                  & 27.32                  & 26.19                  & 25.95                  & 24.77                  & 29.11                  \\ \bottomrule
\end{tabular}
\caption{Intrinsic dimension of CIFAR-10 Dataset. \label{tab:cifar10 individual classes full}}
\end{table*}
\begin{table}[]
\centering
\begin{tabular}{@{}lllllllllll@{}}
\toprule
MNIST class & 0     & 1     & 2     & 3     & 4     & 5     & 6     & 7     & 8     & 9     \\ \midrule
LPCA        & 19    & 8     & 31    & 29    & 27    & 23    & 21    & 21    & 32    & 22    \\
MLE (k=5)   & 13.01 & 9.65  & 14.14 & 15.14 & 13.65 & 14.70 & 12.65 & 12.08 & 15.65 & 12.79 \\
TwoNN       & 15.33 & 12.98 & 15.29 & 16.35 & 14.71 & 15.90 & 14.02 & 13.39 & 16.60 & 14.03 \\ \bottomrule
\end{tabular}
\caption{Estimated intrinsic dimension of MNIST Dataset.\label{tab:mnistclasses}}
\end{table}

\subsection{Simulation}
\subsubsection{Training setting}
\label{appendix: sim train setting}
We use a two-layer ReLU network to train the model. The width $w$ of the neural network is set to be 2000. The initialization of the two layers follows \emph{Kaiming Normal distribution} \citep{he2015delving}. We use gradient descent to train the neural network, i.e., the batch size is $n=1000$, and take the learning rate as 0.1 to train 1000 epochs. For all the settings, the clean training accuracy is 100. We repeat the training process 10 times for an average result and the error bar.
\subsubsection{Minimal strength attack generation}
\label{appendix: minimal strength attack}
In the experiments, we track the attack strength of an unconstrained attack, \emph{on-manifold} and \emph{off-manifold} attack. The way we generate these are described as follows:
\begin{itemize}
    \item Usual $\mathbb{R}^D$ attack: PGD-20 \citep{madry2017towards} attack in the whole space.
    \item On-manifold attack: PGD-20 attack on the intrinsic data space. We achieve this by projecting the PGD algorithm's gradients to the underlying $d$ dimensional intrinsic space by applying the matrix $\bP_o=\bM(\bM^T\bM)^{-1}\bM^T$.
    \item Off-manifold attack: PGD-20 attack on the off-manifold space. We achieve this by projecting the PGD algorithm's gradients to the co-kernel $g$ dimensional space by applying the matrix $\bP=\bI-\bP_o$.
\end{itemize}
\subsubsection{Additional Figures}
\begin{figure*}[!ht]
    \centering
    \includegraphics[scale=0.5]{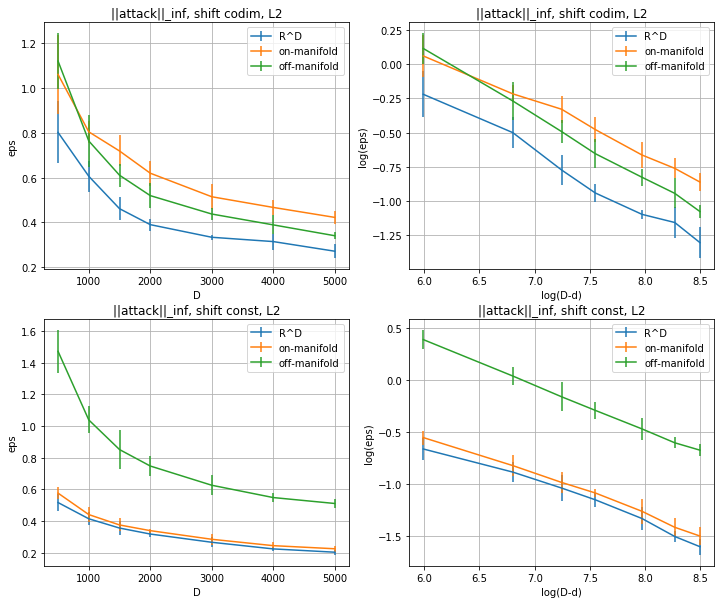}
    \caption{The $\ell_{\infty}$ norm of the $\ell_2$ attack in Figure \ref{fig:change_D}. The top and bottom rows correspond to the cases $\snorm{\bzeta}=\Theta(\sqrt{d}),\Theta(1)$ respectively.}
    \label{fig:simulation:inf}
\end{figure*}
\newpage
\begin{figure*}[!ht]
    \centering
    \includegraphics[scale=0.4]{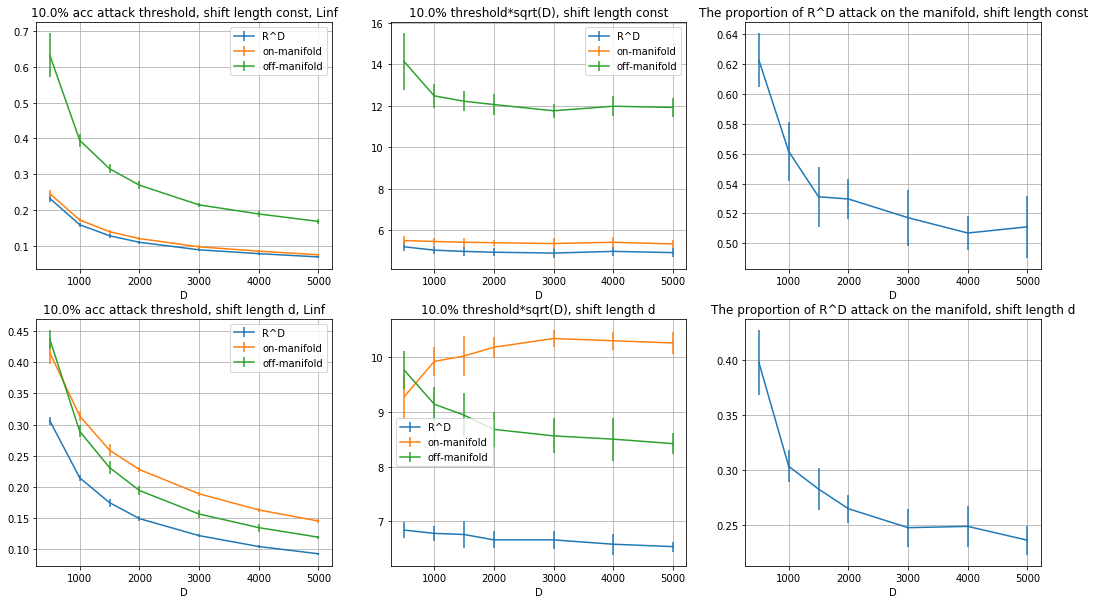}
    \caption{Attack strength threshold associating with $10\%$ robust test accuracy (left), its scaled value (middle), and the on-manifold proportion in $\mathbb{R}^D$ attack. The top and bottom rows correspond to the cases $\snorm{\bzeta}=\Theta(1),\Theta(\sqrt{d})$ respectively.}
    \label{fig:simulation:inf_}
\end{figure*}
\newpage

\begin{figure*}[!ht]
    \centering
    \includegraphics[scale=0.5]{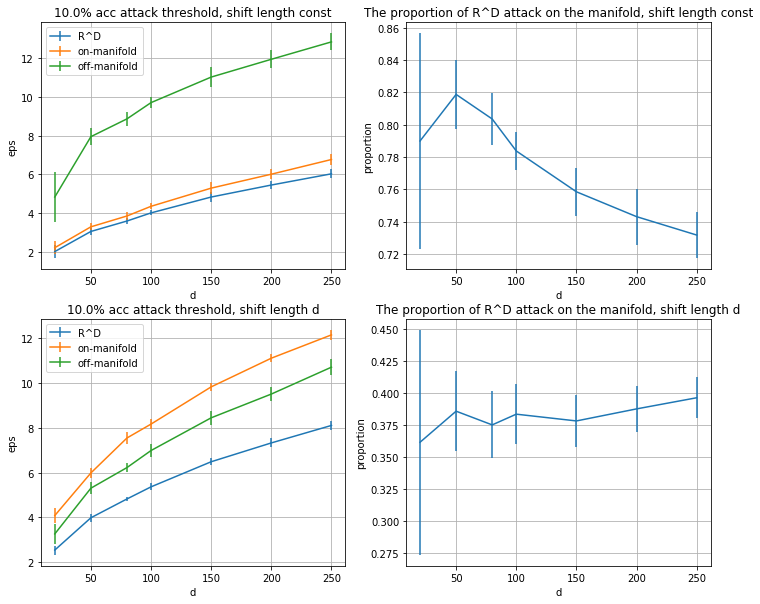}
    \caption{Attack strength threshold associating with $10\%$ robust test accuracy (left) and the on-manifold proportion in $\mathbb{R}^D$ attack when changing $d$. The top and bottom rows correspond to the cases $\snorm{\bzeta}=\Theta(1),\Theta(\sqrt{d})$ respectively.}
    \label{fig:change_codim}
\end{figure*}

\begin{figure*}[!ht]
    \centering
    \includegraphics[scale=0.5]{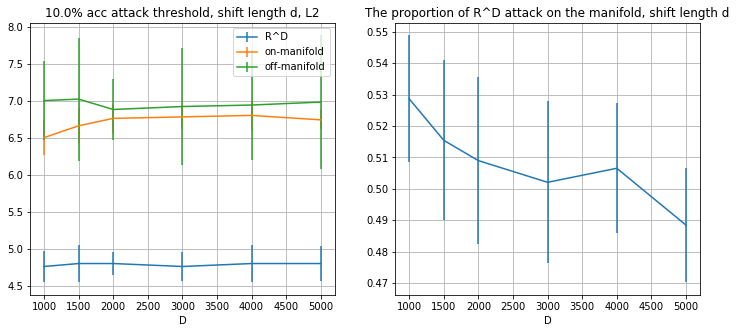}
    \caption{Attack strength threshold associating with $10\%$ robust test accuracy (left) and the on-manifold proportion in $\mathbb{R}^D$ attack when changing $d$. The initialization in this setting is 0.01 of the initialization used in all other figures. }
    \label{fig:change_init}    
\end{figure*}

\begin{figure*}[!ht]
    \centering
    \includegraphics[scale=0.5]{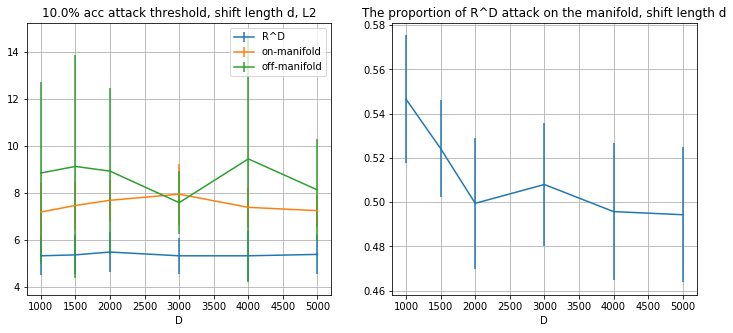}
    \caption{Attack strength threshold associating with $10\%$ robust test accuracy (left) and the on-manifold proportion in $\mathbb{R}^D$ attack when changing $d$. The optimizer takes weight\_decay as 0.1. The number of epochs is changed to 500 in this experiment to save the computation time for large $D$ cases. Since we change the number of epochs in this experiment, we also conduct experiment for the scenario of Figure \ref{fig:change_D}, and there is no change in the result.}
    \label{fig:change_decay}    
\end{figure*}
\FloatBarrier
\subsection{MNIST}
\subsubsection{Architecture and training setting}
\label{appendix: mnist setting}
We use Adam optimizer with learning rate $10^{-3}$ to train the model. The batch size is 256.
\begin{table}[!ht]
    \centering
    \begin{tabular}{c}
         Layer\\\hline
       Conv2d(1, 16, 4, stride=2, padding=1), ReLU\\
      Conv2d(16, 32, 4, stride=2, padding=1), ReLU\\
       Linear(32*(7+$P/2$)*(7+$P/2$),100), ReLU\\
        Linear(100, 10)\\
    \end{tabular}
    \caption{Neural network architecture.}
    \label{tab:my_label}
\end{table}

We take 0.02 as the stopping criteria because when $P=0$, the training loss is around 0.02 after 10 epochs. We repeat the experiment 30 times to get a mean and standard deviation.

\subsubsection{Additional Figures}
\label{appendix: mnist figures}
\begin{figure}[!ht]
    \centering
    \includegraphics[scale=0.5]{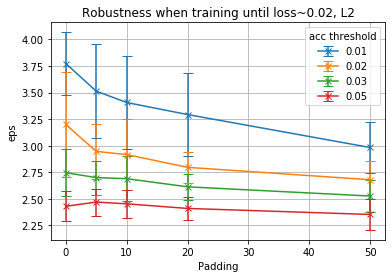}
    \includegraphics[scale=0.5]{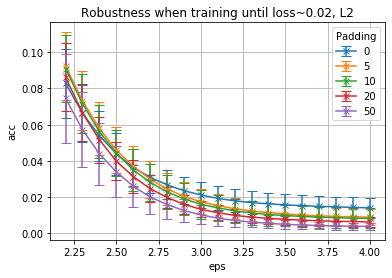}
    \caption{The relationship between attack strength and the padding number $P$ in Table \ref{tab:my_label} (left) and the relationship between robust test accuracy and attack strength under different $P$. $\ell_2$ attack. Train loss is 0.02.} 
    \label{fig:mnist_l2}
\end{figure}

\begin{figure}[!ht]
    \centering
    \includegraphics[scale=0.5]{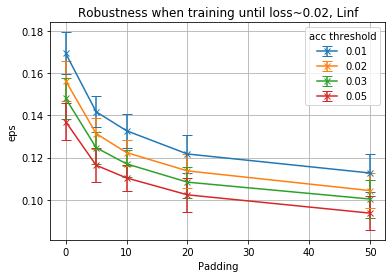}
    \includegraphics[scale=0.5]{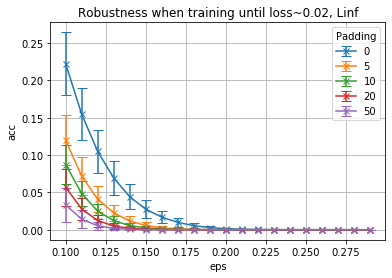}
    \caption{The relationship between attack strength and the padding number $P$ in Table \ref{tab:my_label} (left) and the relationship between robust test accuracy and attack strength under different $P$. $\ell_\infty$ attack. Train loss is 0.02.} 
    \label{fig:mnist_linf}
\end{figure}
\end{document}